\documentclass[letterpaper]{article} 
\usepackage{aaai2026}  
\usepackage{times}  
\usepackage{helvet}  
\usepackage{courier}  
\usepackage[hyphens]{url}  
\usepackage{graphicx} 
\usepackage{natbib}  
\usepackage{caption} 
\usepackage{amsmath,amsfonts,bm}

\def\eqref#1{equation~\ref{#1}}
\def\Eqref#1{Equation~\ref{#1}}

\def\1{\bm{1}}

\def\rd{{\textnormal{d}}}

\def\rvu{{\mathbf{i}}}

\def\rvu{{\mathbf{u}}}

\def\rvx{{\mathbf{x}}}
\def\rvy{{\mathbf{y}}}
\def\rvz{{\mathbf{z}}}
\def\rvZ{{\mathbf{Z}}}

\DeclareMathAlphabet{\mathsfit}{\encodingdefault}{\sfdefault}{m}{sl}
\SetMathAlphabet{\mathsfit}{bold}{\encodingdefault}{\sfdefault}{bx}{n}

\def\sP{{\mathbb{P}}}

\def\sR{{\mathbb{R}}}

\newcommand{\E}{\mathbb{E}}

\usepackage{mathtools}
\usepackage{booktabs} 
\usepackage{algorithm}
\usepackage{amsmath,amsthm,latexsym,amssymb,wasysym}
\usepackage{subcaption}
\usepackage{caption}
\usepackage{multirow}
\usepackage{algpseudocode}
\usepackage{makecell}
\usepackage{mdframed}
\usepackage{array}
\PassOptionsToPackage{numbers,sort,compress}{natbib}
\usepackage{microtype} 
\usepackage{comment}
\usepackage{url}
\usepackage{times}
\usepackage{colortbl}
\usepackage{xcolor}

\theoremstyle{plain}
\newtheorem{theorem}{Theorem}[section]

\newtheorem{corollary}[theorem]{Corollary}
\theoremstyle{definition}
\newtheorem{definition}[theorem]{Definition}

\theoremstyle{remark}

\newcommand{\PreserveBackslash}[1]{\let\temp=\\#1\let\\=\temp}
\newcolumntype{C}[1]{>{\PreserveBackslash\centering}m{#1}}
\newcolumntype{R}[1]{>{\PreserveBackslash\raggedleft}m{#1}}
\newcolumntype{L}[1]{>{\PreserveBackslash\raggedright}m{#1}}

\usepackage{newfloat}
\usepackage{listings}
\DeclareCaptionStyle{ruled}{labelfont=normalfont,labelsep=colon,strut=off} 
\floatstyle{ruled}
\newfloat{listing}{tb}{lst}{}
\floatname{listing}{Listing}
\title{Continuum Dropout for Neural Differential Equations}

\author {
    Jonghun Lee\textsuperscript{\rm 1}\thanks{Equal contribution.},
    YongKyung Oh\textsuperscript{\rm 2}\footnotemark[1],
    Sungil Kim\textsuperscript{\rm 1,3}\thanks{Corresponding authors.},
    Dong-Young Lim\textsuperscript{\rm 1,3}\footnotemark[2]
}
\affiliations {
    \textsuperscript{\rm 1}Artificial Intelligence Graduate School, Ulsan National Institute of Science and Technology (UNIST), Republic of Korea\\
    \textsuperscript{\rm 2}Medical \& Imaging Informatics (MII) Group, University of California, Los Angeles (UCLA), CA, USA\\
    \textsuperscript{\rm 3}Department of Industrial Engineering, Ulsan National Institute of Science and Technology (UNIST), Republic of Korea\\
    jh.lee@unist.ac.kr, yongkyungoh@mednet.ucla.edu, \{sungil.kim, dlim\}@unist.ac.kr
}

\begin{document}

\maketitle

\begin{abstract}
Neural Differential Equations (NDEs) excel at modeling continuous-time dynamics, effectively handling challenges such as irregular observations, missing values, and noise. Despite their advantages, NDEs face a fundamental challenge in adopting dropout, a cornerstone of deep learning regularization, making them susceptible to overfitting. To address this research gap, we introduce Continuum Dropout, a universally applicable regularization technique for NDEs built upon the theory of alternating renewal processes. Continuum Dropout formulates the on-off mechanism of dropout as a stochastic process that alternates between active (evolution) and inactive (paused) states in continuous time. This provides a principled approach to prevent overfitting and enhance the generalization capabilities of NDEs. Moreover, Continuum Dropout offers a structured framework to quantify predictive uncertainty via Monte Carlo sampling at test time. Through extensive experiments, we demonstrate that Continuum Dropout outperforms existing regularization methods for NDEs, achieving superior performance on various time series and image classification tasks. It also yields better-calibrated and more trustworthy probability estimates, highlighting its effectiveness for uncertainty-aware modeling.
\end{abstract}

\begin{links}
    \link{Code}{https://github.com/jonghun-lee0/Continuum-Dropout}
\end{links}

\section{Introduction}\label{sec:Introduction}

Neural Differential Equations (NDEs) have emerged as a powerful framework for modeling continuous-time dynamics by integrating differential equations with neural networks~\citep{chen2018neural, rubanova2019latent, oh2025comprehensive}. This framework has demonstrated superior performance in handling irregular sampling, missing data, and noisy observations across diverse domains such as physics~\citep{greydanus2019hamiltonian}, finance~\citep{yang2023neural}, and others~\citep{rubanova2019latent, kidger2020cde, oh2024stable, oh2025dualdynamics, oh2025tandem}.

\begin{figure*}[h]
    \centering\captionsetup{justification=centering, skip=3pt}
    \captionsetup[subfigure]{justification=centering, skip=3pt}

    \begin{subfigure}{0.22\textwidth}
        \includegraphics[width=\linewidth]{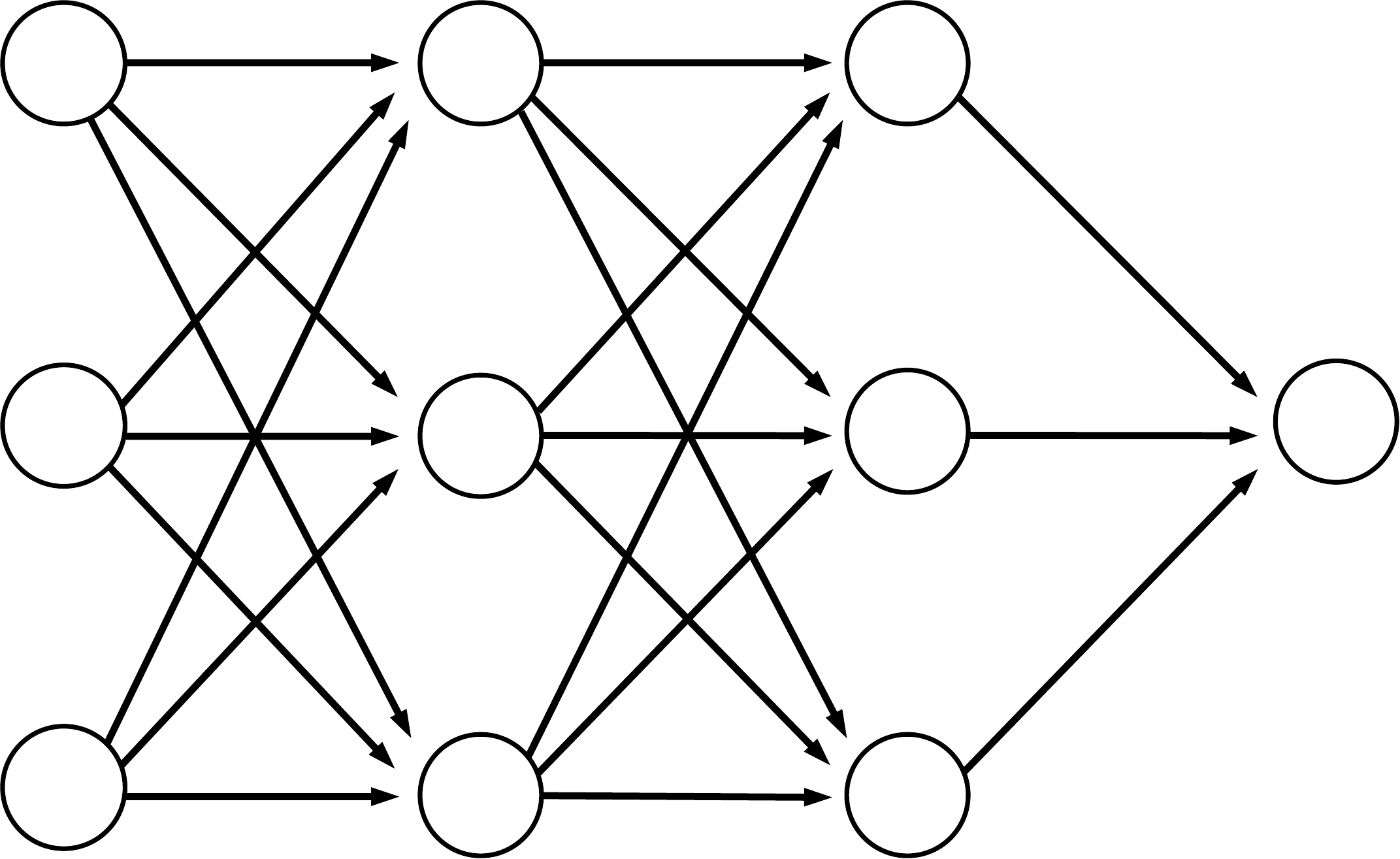}
        \caption{Neural network}
        \label{subfig:nn}
    \end{subfigure}    
    \hspace{0.01\textwidth}
    \begin{subfigure}{0.22\textwidth}
       \includegraphics[width=\linewidth]{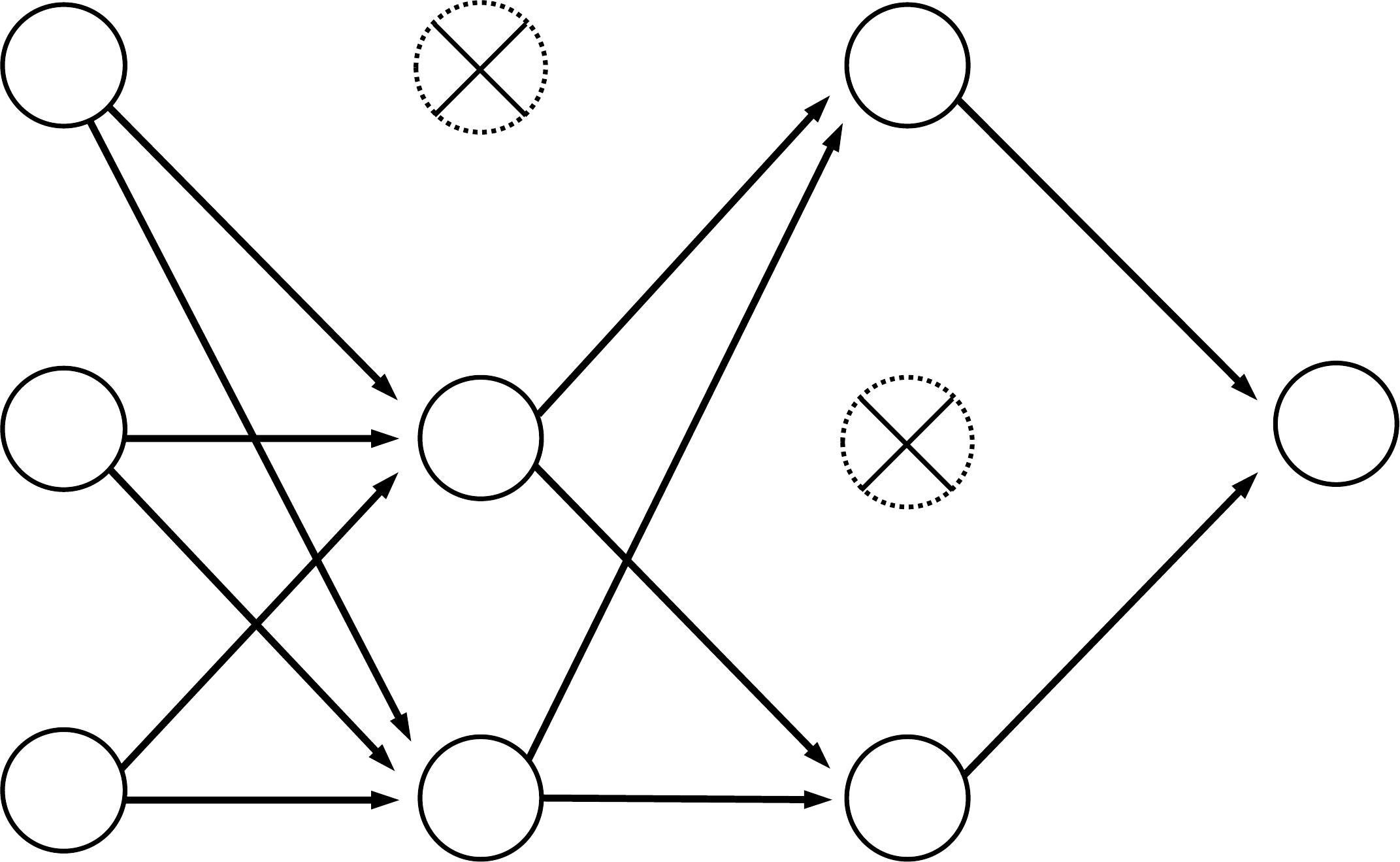}
        \caption{Neural network w/ dropout}
        \label{subfig:nn_dropout}
    \end{subfigure}
    \begin{subfigure}{0.265\textwidth}
       \includegraphics[width=\linewidth]{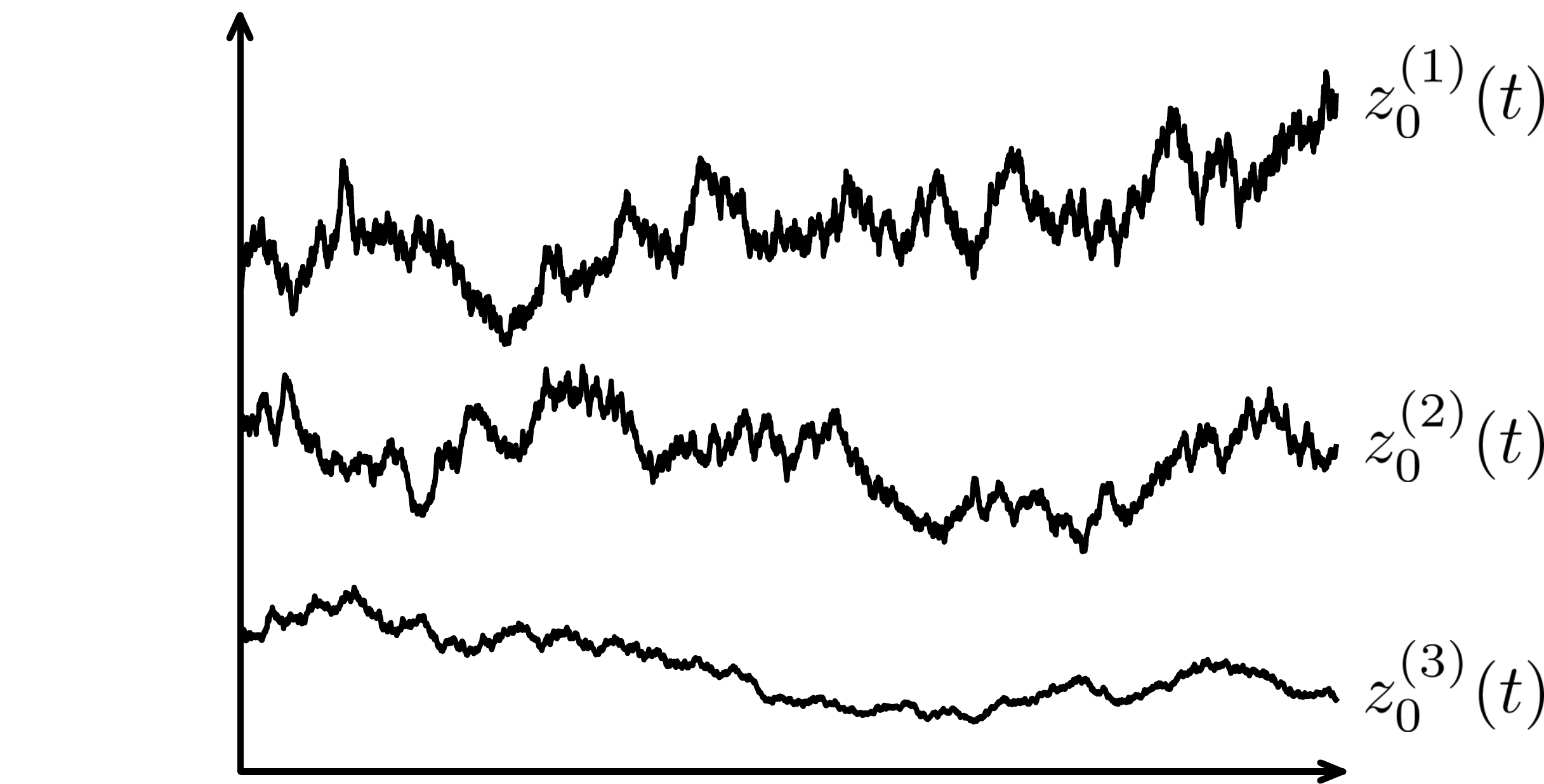}
        \caption{NDE, $\rvz_0(t)$}
        \label{subfig:nde}
    \end{subfigure}
    \begin{subfigure}{0.265\textwidth}
        \includegraphics[width=\linewidth]{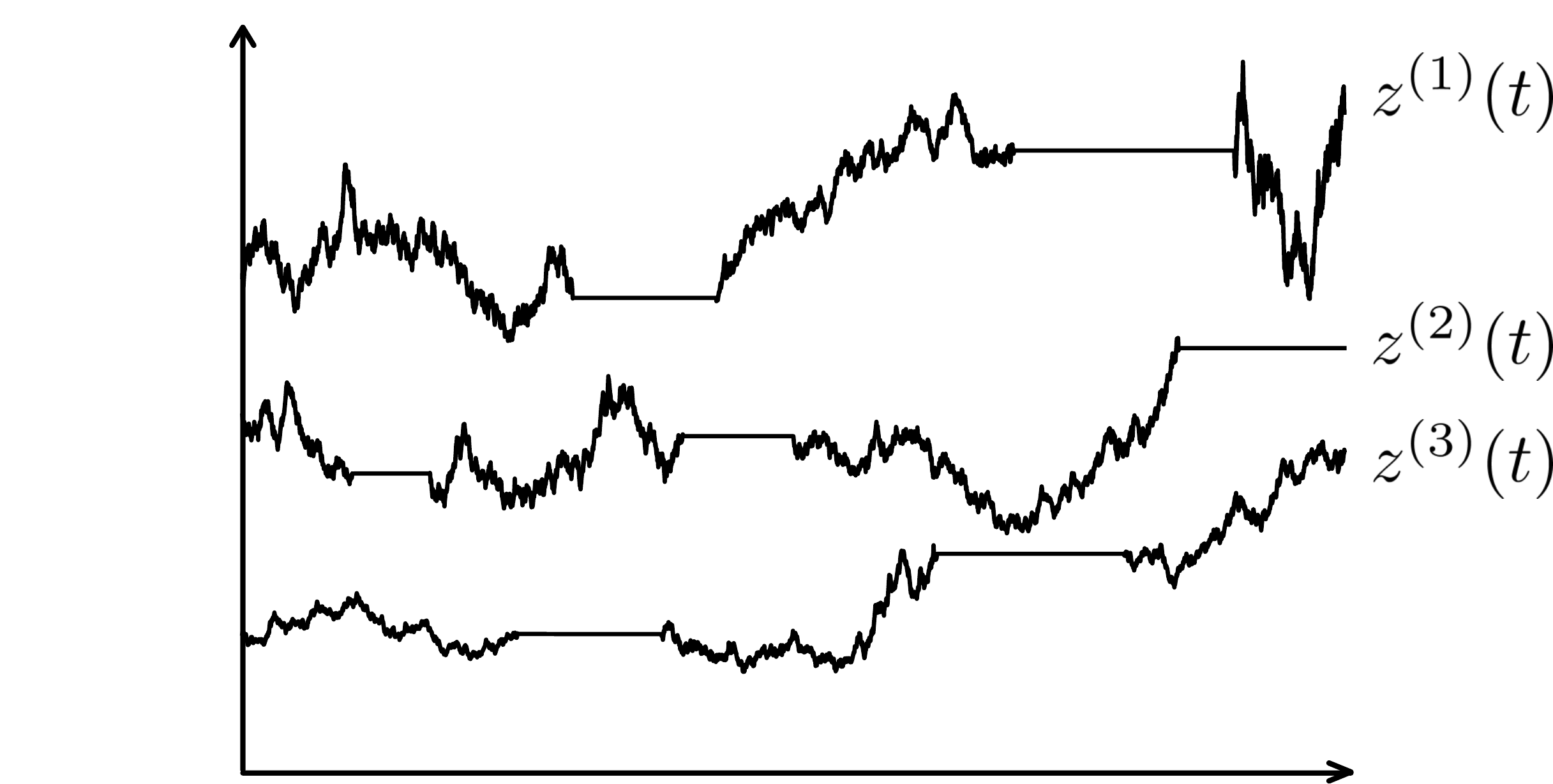}
        \caption{NDE w/ Continuum Dropout, $\rvz(t)$}
        \label{subfig:NRODE}
    \end{subfigure}
    \caption{Illustration of dropout in discrete neural networks and continuous-time latent processes: \\(a) discrete neural network, (b) neural network with dropout, (c) NDE, (d) NDE with Continuum Dropout.}
    \label{fig:comparison}
\end{figure*}

Despite these advantages, NDEs are particularly prone to overfitting, especially when trained on limited data or equipped with complex architectures~\citep{oh2024stable}. To address this issue, various regularization techniques for NDEs have been explored, including Neural Stochastic Differential Equations (Neural SDEs)~\citep{tzen2019neural, kong2020sde, oganesyan2020stochasticity}, Neural Jump Diffusion model~\citep{liu2020does}, STEER~\cite{ghosh2020steer}, Temporal Adaptive Batch Normalization (TA-BN)~\citep{zheng2024improving}, and kinetic energy regularization for Neural ODE-based generative models~\citep{finlay2020train}. However, these methods are often tailored to specific architectures or fail to offer principled uncertainty quantification. Critically, NDEs face a fundamental challenge in adopting dropout~\cite{srivastava2014dropout}, the cornerstone of regularization in discrete neural networks. For instance, na\"{\i}vely applying dropout to the drift function results in an ad-hoc solution that fails to capture the underlying continuous-time structure. While the jump diffusion process has been explored for modeling a continuous-time analogue of dropout~\cite{liu2020does}, the approach is neither a theoretically faithful analogue nor architecturally general, being applicable only to specific types of NDEs. A summary of the limitations of these existing regularization techniques is provided in Table~\ref{tab:methodological_comparison}.

\begin{table}[htb!]
\scriptsize\centering\captionsetup{justification=centering, skip=3pt}
\setlength{\tabcolsep}{3pt}
\caption{Comprehensive methodological comparison of proposed method with key regularization methods for NDEs}
\label{tab:methodological_comparison}
\begin{tabular}{C{1.6cm}C{1.0cm}C{1.0cm}C{1.0cm}C{1.0cm}C{1.2cm}}
\Xhline{0.7pt}
& \textbf{Na\"{\i}ve Dropout} & \textbf{Jump Diffusion} & \textbf{STEER} & \textbf{TA-BN}  & \textbf{Continuum Dropout} \\
\hline
Tailored for NDEs         & $\times$     & $\checkmark$ & $\checkmark$ & $\checkmark$ & $\checkmark$ \\
\hline
Discrete Analogue         & $-$          & $\triangle$  & $\times$     & $\times$     & $\checkmark$ \\
\hline
Quantifying Uncertainty   & $\checkmark$ & $\triangle$  & $\times$     & $\times$     & $\checkmark$ \\
\hline
Architectural Generality  & $\checkmark$ & $\times$     & $\checkmark$ & $\checkmark$ & $\checkmark$ \\
\Xhline{0.7pt}
\end{tabular}
\end{table}

To address this research gap, we introduce Continuum Dropout, a universally applicable regularization technique for NDEs built upon the theory of alternating renewal processes. The key idea is to formulate the on-off mechanism of dropout as a stochastic process that alternates between active (evolution) and inactive (paused) states in continuous time. This naturally leads to a modified NDE where the latent dynamics are randomly turned on and off during training. Moreover, in the spirit of Bayesian dropout~\cite{gal2016dropout} for discrete networks, Continuum Dropout offers a structured framework to quantify predictive uncertainty via Monte Carlo sampling at test time, enabling more trustworthy and uncertainty-aware modeling. Through extensive experiments, we demonstrate that our method not only outperforms existing regularization techniques but also produces better-calibrated probability estimates, highlighting its practical effectiveness.

\section{Preliminaries}
\paragraph{Notations.} For a $\mathbb R^d$-valued vector $\mathbf x$, denote the $i$-th component of $\mathbf x$ by $x^{(i)}$ for $i=1,\ldots, d$. For two matrices $\mathbf{A}$ and $\mathbf{B}$ of the same size, $\mathbf{A} \circ \mathbf{B}$ represents a matrix obtained by element-wise (Hadamard) product of $\mathbf A$ and $\mathbf B$. 
\subsection{Problem Statement}

Let  $\{\rvz_0(t)\}_{0\leq t\leq T}$ be a $d_z$-dimensional continuous-time dynamical (stochastic) process. 
In particular, we consider $\rvz_0(t)$ as a latent process used for various tasks such as prediction, classification, and regression. Let $\mathbf{x}$ denote the $d_x$-dimensional input data, and \(\zeta : \mathbb{R}^{d_x} \rightarrow \mathbb{R}^{d_z} \) is an affine function with parameter \(\theta_{\zeta}\).

To represent the underlying process $\rvz_0(t)$, \citet{chen2018neural} proposed a Neural ODE as the solution of the following ordinary differential equation
\begin{equation}\label{eq:node}
\frac{\rm d \mathbf{z}_0(t)}{\rm d t} = \gamma(t, \mathbf{z}_0(t);\theta_{\gamma})\quad\text{with}\,\,\mathbf{z}_0(0)=\zeta(\mathbf{x};\theta_{\zeta}) ,
\end{equation}
where $0\leq t\leq T$, $\gamma(\cdot; \cdot; \theta_{\gamma})$ is a neural network parameterized by $\theta_{\gamma}$, which is inspired by the following residual connections in ResNet~\cite{he2016deep}:
\begin{equation}\label{eq:resnet}
\begin{aligned}
   \rvZ_{0}(k+1) &= \rvZ_{0}(k) + \gamma(\rvZ_{0}(k); \theta_k),\\
\end{aligned}
\end{equation}
where $\rvZ_0 (k)$ represents the hidden state of ResNet at the $k$-th layer. See \citet{sander2022residual} for a detailed discussion of the relationship between ResNets and Neural ODEs.

On the other hand, when it comes to preventing neural networks from overfitting, the most successful and powerful regularization technique in deep learning models is dropout, which randomly deactivates certain neurons during training~\citep{srivastava2014dropout}. Therefore, to further improve the performance of Neural Differential Equations (NDEs), one naturally pose the following question:

\begin{mdframed}[backgroundcolor=gray!20,linecolor=black]
\textbf{Q.} How can we incorporate the mechanism of dropout into the NDE framework?
\end{mdframed}

The answer to this question lies in developing a continuous-time analogue of a ResNet with dropout: 
\begin{equation}\label{eq:resnet_dropout}
\begin{aligned}
   \rvZ(k+1) &= \rvZ(k) + \mathbf{I}_k \circ \gamma(\rvZ(k); \theta_k),\\
\end{aligned}
\end{equation}
where $\rvZ (k)$ represents the hidden state of ResNet with dropout at the $k$-th layer and $\mathbb{P}(I_k^{(i)} = 0) = 1 - \mathbb{P}(I_k^{(i)} = 1) = p$, with $\mathbf{I}_k \in \mathbb{R}^{d_z}$. Compared to \Eqref{eq:resnet}, \Eqref{eq:resnet_dropout} introduces the Bernoulli variable $\mathbf{I}_k$, which controls the update of the hidden state at step $k+1$. Specifically, when $I_k^{(i)}=1$, the hidden state evolves as in the standard ResNet, whereas when $I_k^{(i)}=0$, its evolution is paused. 

Figure~\ref{fig:comparison} provides a visual comparison of a standard discrete network, a discrete network with dropout, and their continuous-time counterparts. Consider a standard neural network without dropout, as shown in Figure~\ref{subfig:nn}. When dropout is applied, some neurons are temporarily removed from the network, as shown in Figure~\ref{subfig:nn_dropout}. Now, consider the solution $\rvz_0(t)$ of an NDE, which represents the continuous counterpart of the discrete network shown in Figure~\ref{subfig:nde}. Our method constructs a new continuous-time process $\rvz(t)$ that emulates dropout by switching between an active state, where it evolves like $\rvz_0(t)$, and an inactive state, where its evolution is temporarily paused (see Figure~\ref{subfig:NRODE}). This construction serves as the continuous analogue of the ResNet with dropout in \Eqref{eq:resnet_dropout}, effectively incorporating the dropout mechanism into the NDE framework.

Our key insight is that the behavior of dropout in neural networks in continuous-time can be interpreted as a stochastic process where the system alternates between periods of active (evolution) and inactive (paused) states over random time intervals. This type of stochastic behavior is naturally described by alternating renewal processes. Based on this insight, we provide a more robust and theoretically sound regularization technique for NDEs. In addition, we address important follow-up questions: i) what is the proper definition of the dropout rate in continuous-time settings? and ii) how can the dropout rate be controlled throughout training?

\begin{figure*}
    \centering
    \captionsetup{justification=centering, skip=3pt}
    \includegraphics[width=0.82\textwidth]{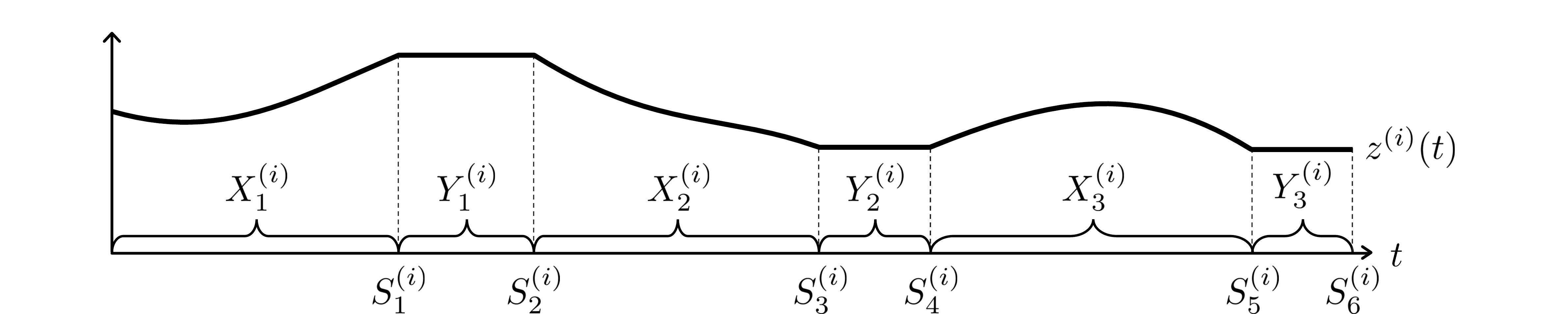}
    \caption{Illustration of $i$-th component of the latent process $\rvz(t)$ with Continuum Dropout}
    \label{fig:alternating}
\end{figure*}

\subsection{Alternating Renewal Process}\label{sec:alternating}

We briefly review the alternating renewal process, which is a key concept in our methodology. This class of regenerative stochastic processes is commonly used to model systems that alternate between two states over time. It has found wide applications in fields such as operations research, queueing theory, and reliability engineering \citep{stanford1979reneging, heath1998heavy, pham1999system, birolini1974some}. For a more detailed explanation, we refer to Appendix~\ref{app:renewal_process}.

Consider a dynamical system that alternates between two (active and inactive) states. Let  $\{X_n\}_{n\geq 1}$ be the sequence of i.i.d. random variables with a distribution $G$ representing the lengths of time that the system is in active state, and let $\{Y_n\}_{n\geq 1}$ be the sequence of i.i.d. random variables with a distribution $H$ representing the lengths of time that the system is in inactive state. Assume that $X_n$ and $Y_m$ are independent for any $n \neq m$. However, $X_n$ and $Y_n$ (for the same index $n$) are not necessarily independent. Moreover, a renewal is defined as the alternation between active and inactive states where the length of each renewal, $\{X_n +Y_n\}_{n\geq 1}$, has a common distribution $F$. Then, the renewal process associated with $\{X_n, Y_n\}_{n\geq1}$ is called an \textit{alternating renewal process}. In particular, we focus on an exponential alternating renewal process where $G$ and $H$ are exponentially distributed. This choice is principled because the memoryless property of the exponential distribution ensures that the state transition probability is independent of the elapsed time, which is consistent with how dropout is applied independently of layer depth in discrete networks.

With $S_0 = 0$, define $S_n$ by for each $n \ge 1$,
\begin{align}\label{eq:arrival_time}
    S_n = \sum_{k=1}^n T_k,
\end{align}
where $T_{2k-1} := X_k$ and $T_{2k} := Y_k$. 
For $t \ge 0$, let $N(t)$ denote the number of renewals in the interval $[0,t]$, i.e.,
\begin{align}\label{eq:counting_proc}
    N(t) = \sum_{n=1}^{\infty} \1_{\{ S_{2n} \le t \}}.
\end{align}
So, $\{N(t)\}_{t \ge 0}$ is the counting process associated with $S_{2n}$.

Here, $S_{2n}$ indicates the arrival time of the $n$-th renewal, while $S_{2n-1}$ denotes 
the start time of the $n$-th inactive state, which lasts for a duration of $Y_n$ until it 
alternates to the next active state. Moreover, at time $t$, the system is in the active state if $S_{2n-2} \le t < S_{2n-1}$, and in the inactive state if $S_{2n-1} \le t < S_{2n}$.

\section{The Proposed Dropout Method for NDEs}\label{sec:methodology}
This section introduces Continuum Dropout. First, we present its mathematical formulation based on alternating renewal processes. Then, we explain its practical application, covering the definition of the dropout rate, hyperparameter tuning, and its use for uncertainty quantification.

\subsection{Continuum Dropout}\label{sec:NRODE}

Let $\rvz_0(t) \in \sR^{d_z}$, be the solution of an Neural ODE given in \Eqref{eq:node}. While we use a Neural ODE for illustrative purposes throughout this section, our framework is universally applicable to any variant of NDEs. For each $i=1,2,\ldots, d_z$, let $\{X_n^{(i)}\}_{n\geq 1}$ and $\{Y_n^{(i)}\}_{n\geq 1}$be sequences of i.i.d. exponential random variables with rates $\lambda_1$ and $\lambda_2$, representing the lengths of active and inactive periods for the $i$-th component of the latent process, respectively. These time periods define an dropout indicator function \( \mathbf{I}_{\lambda_1, \lambda_2}(t) \in \mathbb{R}^{d_z} \), whose components switch between $1$ (active) and $0$ (inactive) according to the alternating renewal process defined in \Eqref{eq:counting_proc} and \Eqref{eq:arrival_time}. That is, the $i$-th component of \( \mathbf{I}_{\lambda_1, \lambda_2}(t) \in \mathbb{R}^{d_z} \) is given by 
\[
I^{(i)}_{\lambda_1, \lambda_2}(t) =
\begin{cases}
1 & \text{if } S_{2n}^{(i)} \leq t < S_{2n+1}^{(i)}, \\
0 & \text{if } S_{2n+1}^{(i)} \leq t < S_{2n+2}^{(i)},
\end{cases}
\quad \text{for all } n \geq 0.
\]
Then, the dynamics of a latent process $\rvz(t)$ with Continuum Dropout are governed by the following differential equation:
\begin{equation}\label{eq:nrode}
\frac{\rd \rvz(t)}{\rd t} = \mathbf{I}_{\lambda_1, \lambda_2}(t) \circ \gamma(t, \rvz(t); \theta_{\gamma}),
\end{equation}
with $\rvz(0) = \rvz_0(0)$. 
The solution to \Eqref{eq:nrode} can be understood through its integral form. During an active period, for $S_{2n}^{(i)} \le t < S_{2n+1}^{(i)}$, the trajectory $z^{(i)}$ shares the same differential equation as the original latent process $z_0^{(i)}$. That is, its solution evolves over this interval as follows:
$$z^{(i)}(t) = z^{(i)}(S_{2n}^{(i)}) + \int_{S_{2n}^{(i)}}^{t} \gamma^{(i)}(s, z(s); \theta_\gamma) \,\rd s.$$
In contrast, during an inactive period, for $S_{2n+1}^{(i)} \le t < S_{2n+2}^{(i)}$, its evolution is paused and the state remains at the fixed value $z^{(i)}(S_{2n+1}^{(i)})$, i.e., $z^{(i)}(t) = z^{(i)}(S_{2n+1}^{(i)})$. Figure 2 provides a visual example of a trajectory $z^{(i)}(t)$.

\subsection{Dropout Rate}\label{subsec:dropout_rate}

In discrete neural networks, the dropout technique has a hyperparameter called the dropout rate $p$. Specifically, if the dropout rate is $p$, then each neuron has a probability of being deactivated for each training iteration. However, in the continuous setting of NDEs, the traditional concept of the dropout rate cannot be directly applicable because each neuron in NDEs represents the value of a continuous-time process at a specific time rather than a countable entity. Therefore, the concept of the dropout rate needs to be redefined to fit the continuous version. 

To this end, we define the continuous-time dropout rate by drawing a direct analogy to its discrete counterpart. In a discrete network, the final prediction is based on the state of the final layer, where each neuron is deactivated with probability $p$. In an NDE, the final prediction is similarly calculated based on the latent state at the terminal time $T$. A natural and consistent definition for the dropout rate $p$ is therefore the probability that a component of the latent process is in the inactive state at this terminal time $T$. This concept is formally known as \textit{instantaneous availability} in renewal theory, which is both theoretically grounded and analytically tractable. Thus, in Continuum Dropout, the dropout rate $p$ is defined as $p=1-A(T)$ where $A(T):= \sP(\{z^{(i)}(T) \mbox{ is in the active state}\})$.

This probability $A(T)$ is closely associated with the intensity parameters $(\lambda_1, \lambda_2)$ of the alternating renewal process, which determine the expected lengths of the active and inactive periods, respectively. Consequently, the dropout rate $p$ can be expressed in terms of $(\lambda_1, \lambda_2)$ in the following theorem. 

\begin{theorem}\label{thm:drop_rate}
For fixed $T>0$, let $\{\rvz(t)\}_{0\leq t \leq T}$ be the latent process with Continuum Dropout where the active and inactive period lengths,  $\{X_n^{(i)}\}_{n\geq 1}$ and  $\{Y_n^{(i)}\}_{n\geq 1}$, are i.i.d. exponential random variables with rates $\lambda_1$ and $\lambda_2$, respectively. Then, the dropout rate $p\in (0,1)$ is given by 
\begin{equation}\label{eq:dropout_rate}
    p = \frac{\lambda_1}{\lambda_1+\lambda_2}\left ( 1 - e^{-(\lambda_1 + \lambda_2)T} \right ).
\end{equation}
\end{theorem}

The proof for Theorem~\ref{thm:drop_rate} can be found in Appendix~\ref{app:proofs}. Given the desired level of dropout rate $p$, there are infinitely many possible pairs $(\lambda_1,\lambda_2)$ that satisfy \Eqref{eq:dropout_rate}. To uniquely determine $(\lambda_1, \lambda_2)$, we impose a second  condition on the expected number of renewals of $\rvz(t)$ in the interval $[0,T]$, denoted as $m:=\E[N^{(i)}(T)]$\footnote{Note that $\E[N^{(1)}(T)]=\E[N^{(2)}(T)]=\cdots=\E[N^{(d_z)}(T)]$ since $\{X_n^{(i)}\}_{n\geq1}$ and $\{Y_n^{(i)}\}_{n\geq1}$ have common distributions $\text{Exp}(\lambda_1)$ and $\text{Exp}(\lambda_2)$, respectively, for all $i$.}. In other words, $m$ represents the average number of repetitions of the active and inactive states in the interval $[0,T]$. Consequently, the user-specified pair $(p,m)$ of Continuum Dropout provides two constraints to uniquely determine the intensity parameters $(\lambda_1, \lambda_2)$, as stated in the following Corollary.

\begin{corollary}\label{cor:drop_rate} Let $p$ be the dropout rate and $m$ be the expected number of renewals of $\rvz(t)$ in the interval $[0,T]$. Given $p\in (0,1)$ and $m>0$, the hyperparameters $\lambda_1$ and $\lambda_2$ of Continuum Dropout can be determined by solving the following system of nonlinear equations:
\begin{equation}\label{eq:pm}
\begin{cases}
\begin{aligned}
p &= \dfrac{\lambda_1}{\lambda_1+\lambda_2}\left ( 1 - e^{-(\lambda_1 + \lambda_2)T} \right ), \\
m &= \dfrac{\lambda_1\lambda_2}{\lambda_1+\lambda_2}\,T - \dfrac{\lambda_1\lambda_2}{(\lambda_1+\lambda_2)^2}\left( 1-e^{-(\lambda_1+\lambda_2)T} \right).
\end{aligned}
\end{cases}
\end{equation}
In particular, for large $T$, $\lambda_1$ and $\lambda_2$ can be approximated by
\begin{equation*}\label{eq:asymp}
\lambda_1 \approx \frac{m}{(1-p)\,T}, \quad \lambda_2 \approx \frac{m}{p\,T}.
\end{equation*}
\end{corollary}

The proof for Corollary~\ref{cor:drop_rate} can be found in Appendix~\ref{app:proofs}. 
While standard dropout has a single hyperparameter, the dropout rate $p$, the proposed Continuum Dropout intrdocues two hyperparameters, $p$ and $m$. Nevertheless, our sensitivity analysis summarized in Appendix~\ref{sec:sensitivity} indicates that the performance of Continuum Dropout is not highly sensitive to the choice of $m$, which reduces the time and effort required for hyperparameter tuning.
Moreover, a detailed explanation of the parameters $(p, m)$ for the trajectories with Continuum Dropout is provided in Figure~\ref{fig:sample_diff_p} of Appendix~\ref{sec:sensitivity}. 

\setlength{\textfloatsep}{10pt}
\begin{algorithm}[!t]
    \small
    \caption{Training and Testing with Continuum Dropout}
    \label{algorithm:dropout}
    \raggedright

    \textbf{Input}: training data $\{\mathbf{x}_i, \mathbf{y}_i\}_{i=1}^{N_{\text{train}}}$, testing data $\{\mathbf{x}^{\dagger}_i\}_{i=1}^{N_{\text{test}}}$, time interval $[0,T]$, hyperparameters $(p, m) \in [0, 1) \times \mathbb{R}^+$, number of epochs $N_{\text{epochs}}$, number of MC simulations $N_{\text{MC}}$
    
    \textbf{Compute}: $(\lambda_1, \lambda_2)$ from $(p, m)$ according to \Eqref{eq:pm}

    \textbf{Initialize}: Network parameters $\theta = [\theta_\zeta, \theta_\gamma, \theta_\sigma, \theta_{\text{MLP}}]$

    \textbf{Training Phase}
    \begin{algorithmic}[1]
        \For{$\text{epoch} = 1$ \textbf{to} $N_{\text{epochs}}$}
            \For{each $(\mathbf{x}, \mathbf{y}) \in \{\mathbf{x}_i, \mathbf{y}_i\}_{i=1}^{N_{\text{train}}}$}
                \State $\mathbf{z}(0) \gets \zeta(\mathbf{x}; \theta_\zeta)$
                \State $\mathbf{z}(T) \gets \text{ODE\_Solve}\Bigl(\mathbf{I}_{\lambda_1,\lambda_2} \circ \gamma_\theta,\mathbf{z}(0),\, [0,T]\Bigr)$\footnotemark
                \State $\mathbf{y}_{\text{pred}} \gets \text{MLP}(\mathbf{z}(T); \theta_{\text{MLP}})$
                \State Compute loss $\mathcal{L}(\mathbf{y}_{\text{pred}}, \mathbf{y})$
            \EndFor
            \State Compute gradients $\nabla_\theta \mathcal{L}$ and update $\theta$ using optimizer
        \EndFor
    \end{algorithmic}

    \textbf{Test Phase}
    \begin{algorithmic}[1]
        \For{each $\mathbf{x}^\dagger \in \{\mathbf{x}_i^\dagger\}_{i=1}^{N_{\text{test}}}$}
            \State $\mathbf{z}(0) \gets \zeta(\mathbf{x}^\dagger; \theta_\zeta)$
            \State Initialize $\mathbf{z}_{\text{sum}} \gets \mathbf{0}$
            \For{$j = 1$ \textbf{to} $N_{\text{MC}}$} \Comment{MC Simulation}
                \State $\mathbf{z}_j(T) \gets \text{ODE\_Solve}\Bigl(\mathbf{I}_{\lambda_1,\lambda_2} \circ \gamma_\theta,\mathbf{z}(0),\, [0,T]\Bigr)$
                \State $\mathbf{z}_{\text{sum}} \gets \mathbf{z}_{\text{sum}} + \mathbf{z}_j(T)$
            \EndFor
            \State $\bar{\mathbf{z}}(T) \gets \mathbf{z}_{\text{sum}}/N_{\text{MC}}$
            \State $\mathbf{y}_{\text{pred}} \gets \text{MLP}(\bar{\mathbf{z}}(T); \theta_{\text{MLP}})$
        \EndFor
        \State \Return $\{\mathbf{y}_{\text{pred}}\}_{i=1}^{N_{\text{test}}}$
    \end{algorithmic}
\end{algorithm}
\footnotetext{For brevity, we denote $\gamma(\cdot, \cdot; \theta_\gamma)$ simply as $\gamma_\theta$.}

\subsection{Uncertainty Quantification in Continuum Dropout}\label{sec:uq}

A key strength of our framework is its natural capacity to quantify predictive uncertainty, particularly epistemic uncertainty, driven by the stochasticity of Continuum Dropout. While many NDEs such as Neural ODE~\citep{chen2018neural} and Neural CDE~\citep{kidger2020cde} generate a single and deterministic latent trajectory for a given input, applying Continuum Droput transforms the dynamics into a stochastic process. This inherent stochasticity, which can be also retained at test time, allows for uncertainty quantification in the spirit of Monte Carlo Dropout~\cite{gal2016dropout}. 

The procedure involves performing $N_{MC}$ stochastic forward passes for a single input. In each pass, a new realization of the dropout indicator function $\mathbf{I}_{\lambda_1, \lambda_2}(t)$ is sampled from the alternating renewal process. Since this indicator function directly modifies the NDE's dynamics, each sample yields a distinct latent trajectory $\{\mathbf{z}_{j}(T)\}_{j=1}^{N_{\text{MC}}}$. This collection forms an empirical distribution over the model's prediction for the given input. Then, we use the predictive mean as the point estimator for the prediction: 
\[
\bar{\mathbf z}(T)\;=\;\frac{1}{N_{\text{MC}}}\sum_{j=1}^{N_{\text{MC}}}\mathbf z_j(T),
\]
and use the sample covariance as a proxy for the model's uncertainty:
\[
\widehat{\operatorname{Var}}\,\!\bigl[\mathbf z(T)\bigr]\;=\;
\frac{1}{N_{\text{MC}}-1}\sum_{j=1}^{N_{\text{MC}}}
\bigl(\mathbf z_j(T)-\bar{\mathbf z}(T)\bigr)
\bigl(\mathbf z_j(T)-\bar{\mathbf z}(T)\bigr)^{\!\top}.
\]
The point estimate $\bar{\mathbf z}(T)$ is then passed through the output MLP to obtain the final prediction. 

As we demonstrate in Section~~\ref{sec:mc_sample}, we empirically observe that as few as $N_{\text{MC}} \ge 5$ simulations are sufficient to obtain stable performance. The complete procedure is summarized in Algorithm~\ref{algorithm:dropout}.

\subsection{Comparison of Alternating Renewal Processes and Jump Diffusion Processes in Dropout Modeling}\label{sec:why?}

While the use of a jump diffusion process to model dropout for NDEs is an interesting direction~\cite{liu2020does}, we argue that it does not serve as a faithful continuous-time analogue. As discussed in Appendix~\ref{app:discretization}, the key issue lies in its discretization: the discrete-time version of their proposed process does not recover the standard dropout formula for ResNets with dropout in \Eqref{eq:resnet_dropout}. This theoretical mismatch suggests a fundamental inconsistency in its modeling of the dropout mechanism. In contrast, the discretization of Continuum Dropout, by construction, aligns with \Eqref{eq:resnet_dropout}, providing a more precise representation of the dropout mechanism.

Beyond theoretical misalignment, alternating renewal processes offer significant practical advantages over the jump diffusion model.  First, the jump diffusion approach is highly sensitive to the dropout rate, improving performance only for very small values ($p<0.1$), which deviates from standard practice and requires extensive tuning (see Appendix~\ref{app:liu_hyperparameter}). On the other hand, Continuum Dropout provides robust performance improvements across a wide range of typical dropout rates. Second, Continuum Dropout consistently outperforms the jump diffusion model in various settings. Lastly, Continuum Dropout is universally applicable to any variant of NDEs. However, the jump diffusion model is limited to specific NDE structures, as detailed in Appendix~\ref{app:why_cannot}.

\section{Experiments}\label{sec:exp}
We perform numerical experiments using real-world datasets to evaluate the effectiveness of the proposed method (Continuum Dropout) in NDEs. First, we compare the proposed method with several regularization methods in Section~\ref{subsec:regularization}. Then, we assess the performance of the proposed method across various NDEs in Section~\ref{subsec:ndes}. In Section \ref{sec:calibration}, we examine Continuum Dropout’s calibration results to verify its uncertainty-aware modeling capability. Finally, Section \ref{sec:mc_sample} analyzes how the number of Monte Carlo samples affects both accuracy and computational cost. For Continuum Dropout, we search for optimal hyperparameters among the dropout rate $p \in \lbrack 0.1, 0.2, 0.3, 0.4, 0.5 \rbrack$ and the expected number of renewals $m \in \lbrack 5, 10, 50, 100 \rbrack$. Throughout our experiments, we use five Monte Carlo samples at test time. A detailed analysis of the impact of MC sample size during the test phase of Continuum Dropout is presented in Section~\ref{sec:mc_sample}. Further details of datasets and experimental settings are summarized in Appendix~\ref{sec:datasets} and Appendix~\ref{sec:settings}, respectively. 

\subsection{Superior Performance Over Existing Regularization Methods}\label{subsec:regularization}

\begin{table*}[htb!]
\scriptsize\centering\captionsetup{justification=centering, skip=3pt}
\caption{Performance of various regularization methods on time series classification}\label{tab:dropout_time}
\begin{tabular}{ccccc}
\Xhline{0.7pt}
\textbf{Regularization Methods}       & \textbf{SmoothSubspace}    & \textbf{ArticularyWordRecognition}  & \textbf{ERing}       & \textbf{RacketSports}      \\ \hline 
Baseline (Neural ODE)   & 0.569 (0.040)       & 0.859 (0.005)    & 0.839 (0.018)     & 0.565 (0.065)         \\ \hline
Na\"{\i}ve Dropout for Drift Network & 0.594 (0.016) & 0.862 (0.014) & 0.844 (0.031) & 0.598 (0.045)\\
Na\"{\i}ve Dropout for MLP Classifier & 0.600 (0.067) & 0.860 (0.006) & 0.856 (0.078) & 0.554 (0.076) \\
Jump Diffusion \cite{liu2020does}   & 0.617 (0.043)   & 0.871 (0.054)  &  0.861 (0.064)  & 0.609 (0.031)        \\
Jump Diffusion + TTN \cite{liu2020does} & 0.606 (0.018) & 0.876 (0.025) & 0.878 (0.025) & 0.598 (0.076)        \\ 
STEER \citep{ghosh2020steer} & 0.578 (0.056) & 0.871 (0.036)  & 0.850 (0.029) & 0.592 (0.054) \\
TA-BN \citep{zheng2024improving} & 0.578 (0.035) & 0.873 (0.023)  & 0.859 (0.023)  & 0.587 (0.035) \\ \hline
\textbf{Continuum Dropout} (ours)  & \textbf{0.629 (0.022)} & \textbf{0.884 (0.012)} & \textbf{0.881 (0.023)} & \textbf{0.619 (0.033)} \\
\Xhline{0.7pt}
\end{tabular}
\end{table*}

\begin{table*}[t]
\scriptsize\centering\captionsetup{justification=centering, skip=3pt}
\caption{Performance of various regularization methods on image classification}\label{tab:dropout}
\begin{tabular}{@{\hspace{0.2cm}}cC{1.8cm}C{1.8cm}C{1.8cm}C{1.8cm}@{\hspace{0.2cm}}}
\Xhline{0.7pt}
\textbf{Regularization Methods}       & \textbf{CIFAR-100}    & \textbf{CIFAR-10}  & \textbf{STL-10}       & \textbf{SVHN}      \\  \hline
Baseline (Neural ODE)   & 0.745 (0.012) & 0.739 (0.008) & 0.707 (0.007) & 0.913 (0.004) \\  \hline
Na\"{\i}ve Dropout for Drift Network & 0.759 (0.004) & 0.749 (0.017) & 0.708 (0.002) & 0.917 (0.004) \\
Na\"{\i}ve Dropout for MLP Classifier & 0.750 (0.006) & 0.751 (0.007) & 0.707 (0.004) & 0.923 (0.001) \\ 
Jump Diffusion \cite{liu2020does}   & 0.761 (0.005) & 0.750 (0.004) & 0.711 (0.002) & 0.916 (0.004) \\
Jump Diffusion + TTN \cite{liu2020does} & 0.760 (0.003) & 0.750 (0.005) & 0.709 (0.003) & 0.917 (0.005) \\ 
STEER \citep{ghosh2020steer} & 0.756 (0.006) & 0.747 (0.008) & 0.710 (0.003) & 0.915 (0.003)\\
TA-BN \citep{zheng2024improving} & 0.758 (0.003) & 0.762 (0.007) & 0.701 (0.007) & 0.915 (0.004) \\  \hline
\textbf{Continuum Dropout} (ours)     & \textbf{0.762 (0.002)} & \textbf{0.765 (0.007)} & \textbf{0.719 (0.005)} & \textbf{0.925 (0.004)} \\ 
\Xhline{0.7pt}
\end{tabular}
\end{table*}

We compare the proposed method (Continuum Dropout) with existing regularization methods for NDEs, using Neural ODE as the baseline. Specifically, we consider \textbf{Na\"{\i}ve Dropout for Drift Network} and \textbf{Na\"{\i}ve Dropout for MLP Classifier}, where na\"{\i}ve dropout \citep{srivastava2014dropout} is applied to the drift network and the MLP classifier, respectively. Additionally, we also consider \textbf{Jump Diffusion model} \cite{liu2020does}, \textbf{STEER} \citep{ghosh2020steer}, and \textbf{Temporal Adaptive Batch Normalization (TA-BN)} \citep{zheng2024improving}. We compute the performance of both with and without dropout noise during the test phase for Jump Diffusion model, following their experimental setup. TTN refers to the use of dropout noise not only during training but also during the test phase. 
Table~\ref{tab:dropout_time} and \ref{tab:dropout} show the classification performance of various regularization methods on time series and image datasets, respectively. We follow the image dataset selection used in previous studies on regularization methods. The performance of each method is evaluated in terms of top-5 accuracy for CIFAR-100 and top-1 accuracy for the other datasets. Continuum Dropout consistently demonstrates superior performance over existing regularization methods across various datasets.

\subsection{Broad Applicability in NDEs}\label{subsec:ndes}

We investigate whether the proposed method (Continuum Dropout) can demonstrate consistent improvements across various NDEs. To evaluate this, we conduct experiments on both time series and image classification tasks. Notably, the time series classification tasks considered here involve challenges such as irregular observations, label imbalance, and missing values where na\"{\i}ve Neural ODE and Neural SDE tend to struggle.

For time series classification tasks, we consider the following NDEs: Neural ODEs (\textbf{GRU-ODE}~\citep{de2019gru}, and \textbf{ODE-RNN}~\citep{rubanova2019latent}), Neural CDEs (\textbf{Neural CDE}~\citep{kidger2020cde} and \textbf{ANCDE}~\citep{jhin2023attentive}),  Neural SDEs (\textbf{Neural LSDE}, \textbf{Neural LNSDE}, and \textbf{Neural GSDE}~\citep{oh2024stable}). Note that Neural LSDE, Neural LNSDE, and Neural GSDE are state-of-the-art models for the time series classification tasks under consideration. We apply the proposed method to NDEs and evaluate its performance. 
Table~\ref{tab:speech} shows classification accuracy of Neural CDE, ANCDE, Neural LSDE, Neural LSDE, Neural LNSDE, Neural GSDE, with and without Continuum Dropout, on Speech Commands. GRU-ODE and ODE-RNN are excluded due to their failure in training~\citep{kidger2020cde}. For PhysioNet Sepsis, Table~\ref{tab:sepsis} presents the AUROC of various NDEs, with and without Continuum Dropout, under different Observation Intensity (OI) settings. Moreover, we mark ``$*$" and ``$**$" to represent the statistical significance of the differences between NDEs with and without Continuum Dropout using two-sample $t$-tests, where $p$-values are less than $0.05$ and $0.01$, respectively. As shown in Table~\ref{tab:speech} and Table~\ref{tab:sepsis}, the proposed method improves the performance of all NDEs, and most of these improvements are statistically significant. In particular, we emphasize that the performance of the current state-of-the-art models, Neural LSDE, Neural LNSDE, and Neural GSDE, can be further improved with proposed method.

\begin{table}[!h]
\scriptsize\centering\captionsetup{justification=centering, skip=3pt}
\caption{Accuracy on Speech Commands}
\label{tab:speech}
\centering
\begin{tabular}{lcc}
\Xhline{0.7pt}
\multirow{2}{*}{\textbf{Models}}                & \multirow{2}{*}{\begin{tabular}[c]{@{}c@{}}\textbf{Continuum} \\ \textbf{Dropout}\end{tabular}}                & \multirow{2}{*}{\textbf{Test Accuracy}} \\ 
 & & \\  \hline
\multirow{2}{*}{Neural CDE}       & X                                 & 0.910 (0.005)   \\
                                  & O                                 & \phantom{$^{**}$}\textbf{0.940 (0.001)}$^{**}$             \\   \hline
\multirow{2}{*}{ANCDE}            & X                                 & 0.760 (0.003)    \\
                                  & O                                 & \phantom{$^{**}$}\textbf{0.794 (0.003)}$^{**}$            \\  \hline  
\multirow{2}{*}{Neural LSDE}      & X                                 & 0.927 (0.004)             \\
                                  & O                                 & \phantom{$^{*}$}\textbf{0.932 (0.000)}$^{*}$    \\  \hline
\multirow{2}{*}{Neural LNSDE}     & X                                 & 0.923 (0.001)             \\
                                  & O                                 & \phantom{$^{**}$}\textbf{0.932 (0.001)}$^{**}$    \\  \hline
\multirow{2}{*}{Neural GSDE}      & X                                 & 0.913 (0.001)             \\
                                  & O                                 & \phantom{$^{**}$}\textbf{0.927 (0.001)}$^{**}$    \\ \Xhline{0.7pt}
\end{tabular}
\end{table}

\begin{table}[!h]
\scriptsize\centering\captionsetup{justification=centering, skip=3pt}
\caption{AUROC on PhysioNet Sepsis}
\label{tab:sepsis}
\centering
\begin{tabular}{@{\hspace{0.15cm}}lccc@{\hspace{0.15cm}}}
\Xhline{0.7pt}
\multirow{2}{*}{\textbf{Models}} & \multirow{2}{*}{\begin{tabular}[c]{@{}c@{}}\textbf{Continuum} \\ \textbf{Dropout}\end{tabular}}   & \multicolumn{2}{c}{\textbf{Test AUROC}}  \\ \cline{3-4} 
                                  &                                   & \textbf{OI}                                      & \textbf{No OI} \\  \hline
\multirow{2}{*}{GRU-ODE}          & X                                 & 0.852 (0.010)                           & 0.771 (0.024)           \\
                                  & O                                 & \phantom{$^{**}$}\textbf{0.875 (0.006)}$^{**}$                  & \phantom{$^{*}$}\textbf{0.808 (0.010)}$^{*}$  \\   \hline 
\multirow{2}{*}{ODE-RNN}          & X                                 & 0.874 (0.016)                           & 0.833 (0.020)           \\
                                  & O                                 & \phantom{$^{*}$}\textbf{0.893 (0.004)}$^{*}$                  & \textbf{0.851 (0.013)}   \\  \hline
\multirow{2}{*}{Neural CDE}       & X                                 & 0.909 (0.006)                           & 0.841 (0.007)           \\
                                  & O                                 & \phantom{$^{*}$}\textbf{0.918 (0.001)}$^{*}$                  & \phantom{$^{**}$}\textbf{0.860 (0.003)}$^{**}$   \\   \hline
\multirow{2}{*}{ANCDE}            & X                                 & 0.900 (0.002)                          & 0.823 (0.003)           \\
                                  & O                                 & \phantom{$^{**}$}\textbf{0.910 (0.001)}$^{**}$                  & \phantom{$^{**}$}\textbf{0.840 (0.005)}$^{**}$   \\   \hline
\multirow{2}{*}{Neural LSDE}      & X                                 & 0.909 (0.004)                           & 0.879 (0.008)           \\
                                  & O                                 & \phantom{$^{**}$}\textbf{0.927 (0.002)}$^{**}$                  & \phantom{$^{**}$}\textbf{0.892 (0.003)}$^{**}$ \\  \hline
\multirow{2}{*}{Neural LNSDE}     & X                                 & 0.911 (0.002)                           & 0.881 (0.002)           \\
                                  & O                                 & \phantom{$^{**}$}\textbf{0.930 (0.001)}$^{**}$                    & \phantom{$^{**}$}\textbf{0.891 (0.003)}$^{**}$ \\  \hline
\multirow{2}{*}{Neural GSDE}      & X                                 & 0.909 (0.001)                           & 0.884 (0.002)           \\
                                  & O                                 & \phantom{$^{**}$}\textbf{0.929 (0.003)}$^{**}$                  & \phantom{$^{**}$}\textbf{0.890 (0.002)}$^{**}$ \\ \Xhline{0.7pt}
\end{tabular}
\end{table}

For image classification tasks, we consider \textbf{Neural ODE}~\citep{chen2018neural} and \textbf{Neural SDEs}~\citep{tzen2019neural, liu2020does, kong2020sde} (Neural SDE with additive noise and Neural SDE with multiplicative noise). Other variants of NDEs are not suitable for image classification since the concept of a controlled path is designed for handling time series data. 

Table~\ref{tab:image} presents the performance of Neural ODE, Neural SDE (additive), and Neural SDE (multiplicative) with and without Continuum Dropout on the four image datasets. The experimental results suggest that proposed method consistently enhances the performance of all NDEs across all datasets with statistical significance. 

\begin{table*}[!h]
\scriptsize\centering\captionsetup{justification=centering, skip=3pt}
\caption{Performance on four image datasets using Neural ODE and Neural SDEs}\label{tab:image}
\begin{tabular}{@{\hspace{0.2cm}}ccC{2cm}C{2cm}C{2cm}C{2cm}@{\hspace{0.2cm}}}
\Xhline{0.7pt}
\multirow{2}{*}{\textbf{Model}} & \multirow{2}{*}{\begin{tabular}[c]{@{}c@{}}\textbf{Continuum} \\ \textbf{Dropout}\end{tabular}} & \multirow{2}{*}{\textbf{CIFAR-100}} & \multirow{2}{*}{\textbf{CIFAR-10}} & \multirow{2}{*}{\textbf{STL-10}} & \multirow{2}{*}{\textbf{SVHN}} \\
 & & & & & \\  \hline
\multirow{2}{*}{Neural ODE} & X & 0.745 (0.006) & 0.739 (0.008) & 0.707 (0.007) & 0.913 (0.004) \\
                           & O & \textbf{\phantom{$^{**}$}0.762 (0.002)$^{**}$} & \textbf{\phantom{$^{**}$}0.765 (0.007)$^{**}$} & \textbf{\phantom{$^{**}$}0.719 (0.005)$^{**}$} & \textbf{\phantom{$^{**}$}0.925 (0.004)$^{**}$} \\  \hline
\multirow{2}{*}{\begin{tabular}[c]{@{}c@{}}Neural SDE\\(additive)\end{tabular}} & X & 0.749 (0.003) & 0.745 (0.009) & 0.707 (0.004) & 0.918 (0.003) \\
                           & O & \textbf{\phantom{$^{**}$}0.770 (0.004)$^{**}$} & \textbf{\phantom{$^{**}$}0.771 (0.006)$^{**}$} & \textbf{\phantom{$^{**}$}0.718 (0.002)$^{**}$} & \textbf{\phantom{$^{**}$}0.925 (0.003)$^{**}$} \\  \hline
\multirow{2}{*}{\begin{tabular}[c]{@{}c@{}}Neural SDE\\(multiplicative)\end{tabular}} & X & 0.753 (0.003) & 0.747 (0.009) & 0.705 (0.005) & 0.914 (0.006) \\
                           & O & \textbf{\phantom{$^{**}$}0.768 (0.004)$^{**}$} & \textbf{\phantom{$^{**}$}0.768 (0.005)$^{**}$} & \textbf{\phantom{$^{**}$}0.715 (0.005)$^{**}$} & \textbf{\phantom{$^{*}$}0.924 (0.003)$^{*}$} \\ \Xhline{0.7pt}
\end{tabular}
\end{table*}

\subsection{Capability for Uncertainty-Aware Modeling}\label{sec:calibration}

While the primary motivation for Continuum Dropout is regularization, its design provides a computationally efficient framework for uncertainty quantification (UQ). As explained in Section~\ref{sec:uq}, we can obtain an empirical predictive distribution via Monte Carlo sampling by retaining the stochastic mechanism during inference. Our goal here is not to claim state-of-the-art performance against specialized Bayesian methods in UQ tasks, but rather to demonstrate that our approach yields a crucial UQ benefit, improved calibration, as a natural byproduct.

To this end, we assess model calibration, where a model's predictive confidence should align with its empirical accuracy. We use reliability diagrams to visualize this alignment; curves that lie close to the identity line indicate well-calibrated and thus more trustworthy uncertainty estimates. As shown in Figure~\ref{fig:calibration}, models with Continuum Dropout produce curves that more closely track the diagonal than their deterministic counterparts. This demonstrates that the injected stochasticity effectively regularizes overconfident predictions, leading to probability estimates that better reflect the model's true predictive uncertainty.

\begin{figure}
    \centering
    \begin{subfigure}[t]{0.48\linewidth}
        \centering
        \includegraphics[width=\linewidth]{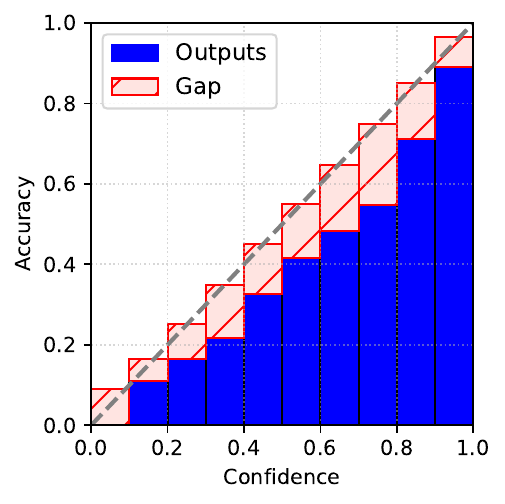}
        \caption{CIFAR-100 without Continuum Dropout}
    \end{subfigure}
    \hfill
    \begin{subfigure}[t]{0.48\linewidth}
        \centering
        \includegraphics[width=\linewidth]{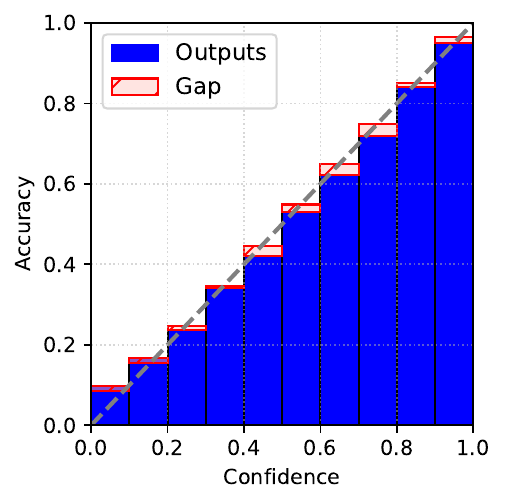}
        \caption{CIFAR-100 with Continuum Dropout}
    \end{subfigure}
    
    \begin{subfigure}[t]{0.48\linewidth}
        \centering
        \includegraphics[width=\linewidth]{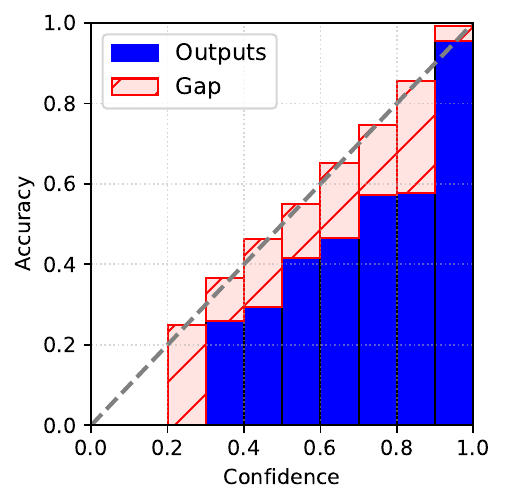}
        \caption{Speech Commands without Continuum Dropout}
    \end{subfigure}
    \hfill
    \begin{subfigure}[t]{0.48\linewidth}
        \centering
        \includegraphics[width=\linewidth]{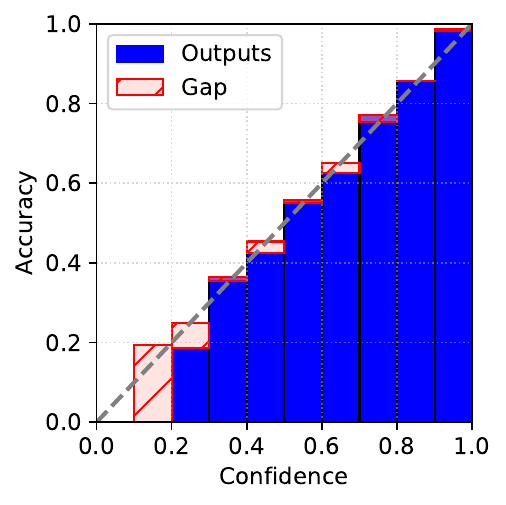}
        \caption{Speech Commands with
        Continuum Dropout}
    \end{subfigure}
    
    \caption{Reliability diagrams illustrating calibration performance on CIFAR-100 and Speech Commands datasets.}
    \label{fig:calibration}
\end{figure}

\subsection{MC Sample Size vs. Performance and Cost}\label{sec:mc_sample}

We conduct experiments to analyze the effect of Monte Carlo (MC) sample size during the test phase of Continuum Dropout, in terms of both model performance and computational cost. First, we apply the proposed method to models such as Neural CDE, ANCDE, Neural LSDE, Neural LNSDE, and Neural GSDE on the Speech Commands dataset, varying the number of MC samples to evaluate performance sensitivity. As shown in Figure~\ref{fig:mcs}, the accuracy of the models stabilizes once the number of MC samples exceeds 5. A similar trend is observed on the CIFAR-100 image dataset, and detailed results are provided in Appendix~\ref{app:c}. 

Next, Table~\ref{tab:cost} reports the computation time in two settings: with and without Continuum Dropout. Although the number of MC sample increases the computational load of the proposed method, the overhead remains within an acceptable range, as MC simulations are performed only during the test phase and require inference on the trained model.

\begin{figure}[h]
    \scriptsize\centering\captionsetup{justification=centering, skip=3pt}
    \includegraphics[width=\columnwidth]{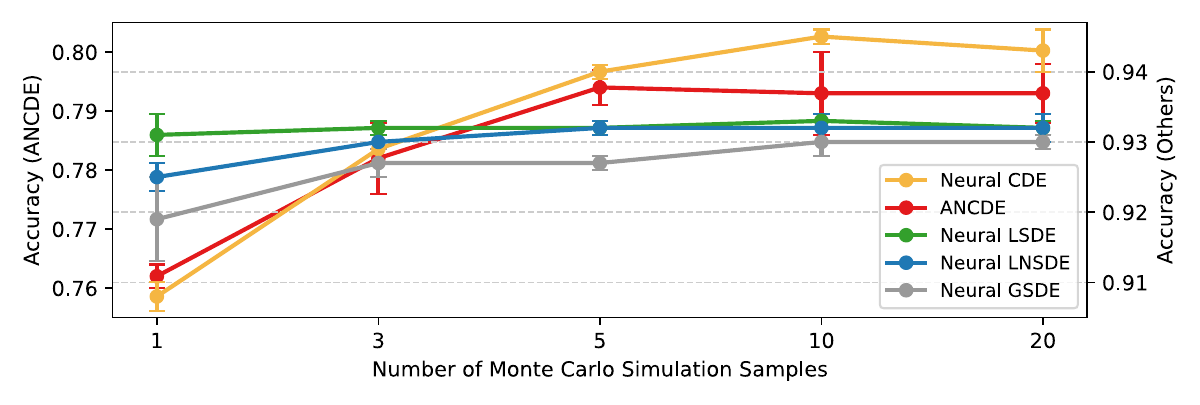} 
    \caption{Performance with Different Numbers of MC Simulation Samples on Speech Commands}
    \label{fig:mcs}
\end{figure}

\begin{table}[h]
\scriptsize\centering\captionsetup{justification=centering, skip=3pt}
\caption{Computation time comparison on Speech Commands (time in seconds per epoch)}
\label{tab:cost}
\begin{tabular}{@{\hspace{0.15cm}}cccccc@{\hspace{0.15cm}}}
\Xhline{0.5pt}
\multirow{2}{*}{\begin{tabular}[c]{@{}c@{}}\textbf{Continuum} \\ \textbf{Dropout}\end{tabular}} & \multirow{2}{*}{\begin{tabular}[c]{@{}c@{}}\textbf{Neural} \\ \textbf{CDE}\end{tabular}} & \multirow{2}{*}{\textbf{ANCDE}} &  \multirow{2}{*}{\begin{tabular}[c]{@{}c@{}}\textbf{Neural} \\ \textbf{LSDE}\end{tabular}} &  \multirow{2}{*}{\begin{tabular}[c]{@{}c@{}}\textbf{Neural} \\ \textbf{LNSDE}\end{tabular}} &  \multirow{2}{*}{\begin{tabular}[c]{@{}c@{}}\textbf{Neural} \\ \textbf{GSDE}\end{tabular}} \\
& & & & & \\   \hline
 X &  25.6 (0.3) & 53.3 (0.2) & 19.4 (0.1) & 19.5 (0.1) & 19.8 (0.1) \\ 
 O & 27.1 (0.3) & 58.9 (0.3) & 21.6 (0.2) & 21.7 (0.2) & 22.5 (0.2) \\
\Xhline{0.7pt}
\end{tabular}
\end{table}

\section{Conclusion}\label{sec:conclusion}
This paper introduced Continuum Dropout, a principled and universally applicable dropout framework for NDEs to enhance their generalization capabilities. Our core contribution is to formulate dropout for NDEs as an alternating renewal process, which provides a faithful continuous-time analogue of the discrete mechanism. We developed a novel, continuous-time definition for the dropout rate and a practical method to control it with intuitive hyperparameters. A key advantage of our framework is that it naturally enables uncertainty quantification via Monte Carlo sampling at test time. Extensive experiments demonstrated that Continuum Dropout consistently outperforms existing regularization methods and is universally applicable across a wide range of NDE architectures, achieving superior results on challenging time-series classification tasks. Furthermore, its ability to enhance model calibration highlights the effectiveness for uncertainty-aware modeling.

\section*{Acknowledgments}
This research was supported by the Institute of Information \& communications Technology Planning \& Evaluation(IITP) grant funded by the Korea government(MSIT)(No.RS-2020-II201336, Artificial Intelligence graduate school support(UNIST)); the Institute of Information \& Communications Technology Planning \& Evaluation(IITP) grant funded by the Korea government(MSIT) (No. RS-2025-25442824, AI STAR Fellowship); the National Research Foundation of Korea(NRF) grant funded by the Korea government(MSIT) (No. RS-2025-02216640); 
the Basic Science Research Program through the National Research Foundation of Korea (NRF) funded by the Ministry of Education (RS-2024-00407852); the 2025 Research Fund (1.250006.01) of UNIST (Ulsan National Institute of Science \& Technology).

\bibliography{references}

@inproceedings{oh2025tandem,
  author = {Oh, YongKyung and Lim, Dong-Young and Kim, Sungil and Bui, Alex A. T.},
  title = {{TANDEM}: {T}emporal {A}ttention-guided {N}eural {D}ifferential {E}quations for {M}issingness in {T}ime {S}eries {C}lassification},
  booktitle = {Proceedings of the 34th ACM International Conference on Information and Knowledge Management},
  series = {CIKM '25},
  year = {2025},
  location = {Seoul, Republic of Korea},
  publisher = {Association for Computing Machinery},
  address = {New York, NY, USA},
  pages = {1--11},
  doi = {10.1145/3746252.3760996},
  isbn = {979-8-4007-2040-6/2025/11},
  note = {Corresponding authors: Sungil Kim and Alex A.~T.~Bui}
}

@article{4,
  title={{O}n fixed gain recursive estimators with discontinuity in the parameters},
  author={Chau, Huy N and Kumar, Chaman and R{\'a}sonyi, Mikl{\'o}s and Sabanis, Sotirios},
  journal={ESAIM: Probability and Statistics},
  volume={23},
  pages={217--244},
  year={2019},
  publisher={EDP Sciences}
}

@article{chen2018neural,
  title={{N}eural {O}rdinary {D}ifferential {E}quations},
  author={Chen, Ricky T. Q. and Rubanova, Yulia and Bettencourt, Jesse and Duvenaud, David},
  journal={Advances in Neural Information Processing Systems},
  volume={31},
  year={2018}
}

@inproceedings{oh2024stable,
  title={{S}table {N}eural {S}tochastic {D}ifferential {E}quations in {A}nalyzing {I}rregular {T}ime {S}eries {D}ata},
  author={Oh, YongKyung and Lim, Dongyoung and Kim, Sungil},
  booktitle={The Twelfth International Conference on Learning Representations},
  year={2024}
}

@article{rubanova2019latent,
  title={{L}atent {ODE}s for {I}rregularly-{S}ampled {T}ime {S}eries},
  author={Rubanova, Yulia and Chen, Ricky T. Q. and Duvenaud, David},
  journal={Advances in Neural Information Processing Systems},
  volume={32},
  year={2019}
}

@article{greydanus2019hamiltonian,
  title={{H}amiltonian {N}eural {N}etworks},
  author={Greydanus, Sam and Dzamba, Misko and Yosinski, Jason},
  journal={Advances in Neural Information Processing Systems},
  volume={32},
  year={2019}
}

@article{yang2023neural,
  title={{N}eural network stochastic differential equation models with applications to financial data forecasting},
  author={Yang, Luxuan and Gao, Ting and Lu, Yubin and Duan, Jinqiao and Liu, Tao},
  journal={Applied Mathematical Modelling},
  volume={115},
  pages={279--299},
  year={2023},
  publisher={Elsevier}
}

@article{srivastava2014dropout,
  title={{D}ropout: {A} {S}imple {W}ay to {P}revent {N}eural {N}etworks from {O}verfitting},
  author={Srivastava, Nitish and Hinton, Geoffrey and Krizhevsky, Alex and Sutskever, Ilya and Salakhutdinov, Ruslan},
  journal={The Journal of Machine Learning Research},
  volume={15},
  pages={1929--1958},
  year={2014},
  publisher={JMLR. org}
}

@inproceedings{liu2020does,
  title={{H}ow {D}oes {N}oise {H}elp {R}obustness? {E}xplanation and {E}xploration under the {N}eural {SDE} {F}ramework},
  author={Liu, Xuanqing and Xiao, Tesi and Si, Si and Cao, Qin and Kumar, Sanjiv and Hsieh, Cho-Jui},
  booktitle={IEEE/CVF Conference on Computer Vision and Pattern Recognition (CVPR)}, 
  year={2020},
  volume={},
  number={},
  pages={279-287}
}

@article{oganesyan2020stochasticity,
  title={{S}tochasticity in {N}eural {ODE}s: {A}n {E}mpirical {S}tudy},
  author={Oganesyan, Viktor and Volokhova, Alexandra and Vetrov, Dmitry},
  journal={arXiv preprint arXiv:2002.09779},
  year={2020}
}

@inproceedings{gal2016dropout,
  title={{D}ropout as a {B}ayesian {A}pproximation: {R}epresenting {M}odel {U}ncertainty in {D}eep {L}earning},
  author={Gal, Yarin and Ghahramani, Zoubin},
  booktitle={International Conference on Machine Learning},
  pages={1050--1059},
  volume = {48},
  year={2016},
  publisher = {PMLR}
}

@inproceedings{finlay2020train,
  title={{H}ow to {T}rain {Y}our {N}eural {ODE}: the {W}orld of {J}acobian and {K}inetic {R}egularization},
  author={Finlay, Chris and Jacobsen, J{\"o}rn-Henrik and Nurbekyan, Levon and Oberman, Adam M},
  booktitle={International conference on machine learning},
  pages={3154--3164},
  year={2020},
  organization={PMLR}
}

@article{tzen2019neural,
  title={{N}eural {S}tochastic {D}ifferential {E}quations: {D}eep {L}atent {G}aussian {M}odels in the {D}iffusion {L}imit},
  author={Tzen, Belinda and Raginsky, Maxim},
  journal={arXiv preprint arXiv:1905.09883},
  year={2019}
}

@inproceedings{reyna2019early,
  title={{E}arly {P}rediction of {S}epsis from {C}linical {D}ata: {T}he {P}hysio{N}et/{C}omputing in {C}ardiology {C}hallenge 2019},
  author={Reyna, Matthew A and Josef, Chris and Seyedi, Salman and Jeter, Russell and Shashikumar, Supreeth and Moody, Benjamin and Westover, M. Brandon and Sharma, Ashish and Nemati, Shamim and Clifford, Gari D.},
  booktitle={2019 Computing in Cardiology (CinC)},
  pages={1--4},
  year={2019},
  organization={IEEE}
}

@article{warden2018speech,
  title={{S}peech {C}ommands: {A} {D}ataset for {L}imited-{V}ocabulary {S}peech {R}ecognition},
  author={Warden, Pete},
  journal={arXiv preprint arXiv:1804.03209},
  year={2018}
}

@book{ross1995stochastic,
  title={{S}tochastic {P}rocesses},
  author={Sheldon M. Ross},
  year={1995},
  publisher={Wiley}
}

@article{kidger2020cde,
  title={{N}eural {C}ontrolled {D}ifferential {E}quations for {I}rregular {T}ime {S}eries},
  author={Kidger, Patrick and Morrill, James and Foster, James and Lyons, Terry},
  journal={Advances in Neural Information Processing Systems},
  volume={33},
  pages={6696--6707},
  year={2020}
}

@article{ghosh2020steer,
  title={{STEER}: {S}imple {T}emporal {R}egularization {F}or {N}eural {ODE}s},
  author={Ghosh, Arnab and Behl, Harkirat Singh and Dupont, Emilien and Torr, Philip H. S. and Namboodiri, Vinay},
  journal={Advances in Neural Information Processing Systems},
  volume={33},
  pages={6696--6707},
  year={2020}
}

@article{pham1999system,
  title={{S}ystem {A}vailability in a {G}amma {A}lternating {R}enewal {P}rocess},
  author={Pham‐Gia, T. and Turkkan, N.},
  journal={Naval Research Logistics (NRL)},
  volume={46},
  number={7},
  pages={822--844},
  year={1999},
  publisher={John Wiley \& Sons}
}

@misc{krizhevsky2009learning,
  title={{L}earning {M}ultiple {L}ayers of {F}eatures from {T}iny {I}mages},
  author={Krizhevsky, Alex and Hinton, Geoffrey},
  year={2009},
  publisher={Toronto, ON, Canada}, 
  
}

@inproceedings{coates2011analysis,
  title={{A}n {A}nalysis of {S}ingle-{L}ayer {N}etworks in {U}nsupervised {F}eature {L}earning},
  author={Coates, Adam and Ng, Andrew and Lee, Honglak},
  booktitle={Proceedings of the fourteenth international conference on artificial intelligence and statistics},
  pages={215--223},
  year={2011}
}

@article{netzer2011reading,
  title={{R}eading {D}igits in {N}atural {I}mages with {U}nsupervised {F}eature {L}earning},
  author={Netzer, Yuval and Wang, Tao and Coates, Adam and Bissacco, Alessandro and Wu, Baolin and Ng, Andrew Y},
  journal={NIPS workshop on deep learning and unsupervised feature learning},
  volume={2011},
  number={5},
  pages={7},
  year={2011},
}

@article{huang2016time,
  title={{T}ime series k-means: {A} new k-means type smooth subspace clustering for time series data},
  author={Huang, Xiaohui and Ye, Yunming and Xiong, Liyan and Lau, Raymond YK and Jiang, Nan and Wang, Shaokai},
  journal={Information Sciences},
  volume={367},
  pages={1--13},
  year={2016},
  publisher={Elsevier}
}

@article{stanford1979reneging,
  title={{R}eneging {P}henomena in {S}ingle {C}hannel {Q}ueues},
  author={Stanford, Robert E.},
  journal={Mathematics of Operations Research},
  volume={4},
  number={2},
  year={1979},
  pages={162--178},
}

@article{heath1998heavy,
author = {Heath, David and Resnick, Sidney and Samorodnitsky, Gennady},
title = {{H}eavy {T}ails and {L}ong {R}ange {D}ependence in {O}n/{O}ff {P}rocesses and {A}ssociated {F}luid {M}odels},
journal = {Mathematics of Operations Research},
volume = {23},
number = {1},
pages = {145-165},
year = {1998},
}

@book{cox1962renewal,
  title={{R}enewal {T}heory},
  author={D. R. Cox},
  year={1962},
  publisher={Methuen}
}

@inproceedings{kong2020sde,
  title={{SDE}-{N}et: {E}quipping {D}eep {N}eural {N}etworks with {U}ncertainty {E}stimates},
  author={Kong, Li and Sun, Jiawei and Zhang, Cheng},
  booktitle={International Conference on Machine Learning},
  volume={37},
  year={2020}
}

@ARTICLE{birolini1974some,
  author={Birolini, A.},
  journal={IEEE Transactions on Reliability}, 
  title={{S}ome {A}pplications of {R}egenerative {S}tochastic {P}rocesses to {R}eliability {T}heory - {P}art {O}ne: {T}utorial {I}ntroduction}, 
  year={1974},
  volume={R-23},
  number={3},
  pages={186-194},
}

@inproceedings{de2019gru,
  title={{GRU}-{ODE}-{B}ayes: {C}ontinuous modeling of sporadically-observed time series},
  author={De Brouwer, Edward and Simm, Jaak and Arany, Adam and Moreau, Yves},
  booktitle={Advances in neural information processing systems},
  volume={32},
  year={2019}
}

@article{jhin2023attentive,
  title={{A}ttentive {N}eural {C}ontrolled {D}ifferential {E}quations for {T}ime-series {C}lassification and {F}orecasting},
  author={Jhin, Sheo Yon and Shin, Heejoo and Kim, Sujie and Hong, Seoyoung and Jo, Minju and Park, Solhee and Park, Noseong and Lee, Seungbeom and Maeng, Hwiyoung and Jeon, Seungmin},
  journal={Knowledge and Information Systems},
  pages={1--31},
  year={2023}
}

@inproceedings{he2016deep,
  title={Deep residual learning for image recognition},
  author={He, Kaiming and Zhang, Xiangyu and Ren, Shaoqing and Sun, Jian},
  booktitle={Proceedings of the IEEE conference on computer vision and pattern recognition},
  pages={770--778},
  year={2016}
}

@inproceedings{moritz2018ray,
  title={Ray: A distributed framework for emerging {AI} applications},
  author={Moritz, Philipp and Nishihara, Robert and Wang, Stephanie and Tumanov, Alexey and Liaw, Richard and Liang, Eric and Elibol, Melih and Yang, Zongheng and Paul, William and Jordan, Michael I and others},
  booktitle={13th USENIX Symposium on Operating Systems Design and Implementation OSDI 18},
  pages={561--577},
  year={2018}
}

@article{liaw2018tune,
    title={{T}une: {A} {R}esearch {P}latform for {D}istributed {M}odel {S}election and {T}raining},
    author={Liaw, Richard and Liang, Eric and Nishihara, Robert
            and Moritz, Philipp and Gonzalez, Joseph E and Stoica, Ion},
    journal={arXiv preprint arXiv:1807.05118},
    year={2018}
}

@article{bagnall2018uea,
  title={The {UEA} multivariate time series classification archive, 2018},
  author={Bagnall, Anthony and Dau, Hoang Anh and Lines, Jason and Flynn, Michael and Large, James and Bostrom, Aaron and Southam, Paul and Keogh, Eamonn},
  journal={arXiv preprint arXiv:1811.00075},
  year={2018}
}

@article{loning2019sktime,
  title={sktime: A unified interface for machine learning with time series},
  author={L{\"o}ning, Markus and Bagnall, Anthony and Ganesh, Sajaysurya and Kazakov, Viktor and Lines, Jason and Kir{\'a}ly, Franz J},
  journal={arXiv preprint arXiv:1909.07872},
  year={2019}
}

@article{dau2019ucr,
  title={The {UCR} time series archive},
  author={Dau, Hoang Anh and Bagnall, Anthony and Kamgar, Kaveh and Yeh, Chin-Chia Michael and Zhu, Yan and Gharghabi, Shaghayegh and Ratanamahatana, Chotirat Ann and Keogh, Eamonn},
  journal={IEEE/CAA Journal of Automatica Sinica},
  volume={6},
  number={6},
  pages={1293--1305},
  year={2019},
  publisher={IEEE}
}

@inproceedings{wang2013word,
  title={{W}ord {R}ecognition from {C}ontinuous {A}rticulatory {M}ovement {T}ime-{S}eries {D}ata
using {S}ymbolic {R}epresentations},
  author={Wang, Jun and Balasubramanian, Arvind and de la Vega, Luis Gerardo Mojica and Green, Jordan R and Samal, Ashok and Prabhakaran, Balakrishnan},
  booktitle={Proceedings of the Fourth Workshop on Speech and Language Processing for Assistive Technologies},
  pages={119--127},
  year={2013}
}

@inproceedings{wilhelm2015ering,
  title={e{R}ing: {M}ultiple {F}inger {G}esture {R}ecognition with one {R}ing
{U}sing an {E}lectric {F}ield},
  author={Wilhelm, Mathias and Krakowczyk, Daniel and Trollmann, Frank and Albayrak, Sahin},
  booktitle={Proceedings of the 2nd international Workshop on Sensor-based Activity Recognition and Interaction},
  pages={1--6},
  year={2015}
}

@article{sander2022residual,
  title={{D}o {R}esidual {N}eural {N}etworks {D}iscretize {N}eural {O}rdinary {D}ifferential {E}quations?},
  author={Sander, Michael and Ablin, Pierre and Peyr{\'e}, Gabriel},
  journal={Advances in Neural Information Processing Systems},
  volume={35},
  pages={36520--36532},
  year={2022}
}

@inproceedings{zheng2024improving,
  title={{I}mproving {N}eural {ODE} {T}raining with {T}emporal {A}daptive {B}atch {N}ormalization},
  author={Zheng, Su and Gao, Zhengqi and Sun, Fan-Keng and Boning, Duane S and Yu, Bei and Wong, Martin},
  booktitle={The Thirty-eighth Annual Conference on Neural Information Processing Systems},
  year={2024}
}

@misc{whitt2013IEOR3106,
  author       = {Ward Whitt},
  title        = {IEOR~3106: Alternating Renewal Processes and the Renewal Equation},
  howpublished = {Lecture notes, Department of Industrial Engineering and Operations Research, Columbia University},
  year         = {2013},
  url          = {https://www.columbia.edu/~ww2040/3106F13/lect1119.pdf},
}

@inproceedings{oh2025comprehensive,
  title={{C}omprehensive {R}eview of {N}eural {D}ifferential {E}quations for {T}ime {S}eries {A}nalysis},
  author={Oh, YongKyung and Kam, Seungsu and Lee, Jonghun and Lim, Dong-Young and Kim, Sungil and Bui, Alex A. T.},
  booktitle = {Proceedings of the Thirty-Fourth International Joint Conference on
               Artificial Intelligence},
  pages     = {10621--10631},
  year      = {2025},
  month     = {8},
}

@article{oh2025dualdynamics,
  title={{D}ual{D}ynamics: {S}ynergizing {I}mplicit and {E}xplicit {M}ethods for {R}obust {I}rregular {T}ime {S}eries {A}nalysis},
  author={Oh, YongKyung and Lim, Dong-Young and Kim, Sungil},
  journal={Proceedings of the AAAI Conference on Artificial Intelligence},
  volume={39},
  number={18},
  pages={19730--19739},
  year={2025}
}

\setcounter{secnumdepth}{2} 

\renewcommand{\thesection}{\Alph{section}}  
\renewcommand{\thesubsection}{\thesection.\arabic{subsection}} 

\onecolumn

\clearpage
\newpage
\appendix

\section{Overview of Alternating Renewal Process}\label{app:renewal_process}
The renewal process is a concept generalized from the Poisson process, where interarrival times follow not only the exponential distribution but also more general distributions. This process is commonly used for analyzing purchase cycles, replacement times, and similar phenomena and has been extensively studied under the name renewal theory. Among these, we focus on the alternating renewal process, which is a renewal process where a system alternates between two states, `on' and `off'. In this appendix, we provide a brief overview of renewal theory. For more details, we refer to \citep{ross1995stochastic, cox1962renewal, pham1999system}.

Let $F$ denote the common distribution of the independent interarrival times $\{T_n\}_{n \geq 1}$,
\begin{equation*}
F(t) = \mathbb{P}(T_1 \leq t), \quad  t \in [0, \infty).
\end{equation*}

And let \(\mu\) be the mean time of \(T_1\), i.e.,
\begin{equation*}
\mu = \mathbb{E}[T_1] = \int_{0}^{\infty }t\,\mathrm{d}F(t).
\end{equation*}

Define $S_n$ as the time of the $n$-th arrival, given by 
\begin{align*}
    S_0=0, \quad S_n = \sum_{k=1}^n \,T_k, \quad n\geq 1.
\end{align*}

For $n \in \mathbb{N}$, let $F_n$ denote the distribution of $S_n$, 
\begin{equation*}
F_n(t) = \mathbb{P}(S_n \leq t), \quad  t \in [0, \infty).
\end{equation*}

Let $N(t)$ denote the number of arrivals in time $[0, t]$, 
\begin{equation*}
N(t) = \sum_{n=1}^{\infty } \1_{\{S_{n} \leq t\}} = \sup_{n \geq 0} \{ S_{n} \leq t \}
\end{equation*}

\begin{definition}
The counting process $\{N(t)\}_{t \geq 0}$ is called a renewal process.
\end{definition}

The distribution of $N(t)$ can be obtained as
\begin{equation*}
\begin{aligned}
\mathbb{P}(N(t) = n) &= \mathbb{P}(N(t) \geq n) - \mathbb{P}(N(t) \geq n + 1) \\
&= \mathbb{P}(S_n \leq t) - \mathbb{P}(S_{n+1} \leq t) \\
&= F_n(t) - F_{n+1} (t).
\end{aligned}
\end{equation*}

The renewal function $m(t)$ is defined as the expected value of $N(t)$, i.e., $m(t) = \mathbb{E}[N(t)]$, is closely related to most renewal theory and determines the properties of the renewal process.

\begin{theorem}\label{thm:mFn}
Let $m(t)$ be the renewal function of $N(t)$, defined as $m(t) = \mathbb{E}[N(t)]$. Then, $m(t)$ is given in terms of the arrival distribution function by
\begin{equation*}
m(t) = \displaystyle\sum_{n=1}^{\infty }F_n(t).
\end{equation*}
\end{theorem}

We summarize some important theorems related to the renewal function $m(t)$.

\begin{theorem}\label{thm:renewal_eq}
(The Renewal Equation~\citep{ross1995stochastic}) Let the distribution $F$ has the density $f$, then
\begin{equation*}
    m(t) = F(t) + \int_{0}^{t}f(t-\tau)m(\tau)\,\mathrm{d}\tau,
\end{equation*}
\end{theorem}

\begin{theorem}\label{thm:element}
(The Elementary Renewal Theorem~\citep{ross1995stochastic}) For finite $\mu$,
\begin{equation*}
\frac{m(t)}{t} \to \frac{1}{\mu} \quad \text{as } \,t \to \infty.
\end{equation*}
\end{theorem}

\begin{theorem}
(The Key Renewal Theorem~\citep{ross1995stochastic}) Suppose that the renewal process is aperiodic and $g : [0, \infty) \rightarrow [0, \infty)$ is directly Riemann integrable, then
\begin{equation*}
\int_{0}^{t} g(t-\tau)\,\mathrm{d}m(\tau) \to \frac{1}{\mu}\int_{0}^{\infty }g(\tau) \,\mathrm{d}\tau \quad \text{as } \,t \to \infty.
\end{equation*}
\end{theorem}

%
%
%
%

Now, we define an alternating renewal process. Consider a system that alternates two states, `on' and `off'. It starts with the state `on' for a time $X_1$, and then transitions to the state `off' for a time $Y_1$. This pattern continues with $X_2$ in the `on' state, followed by $Y_2$ in the `off' state, and so forth. Suppose that $(X_n, Y_n)$ for $n\geq 1$ are independent random vectors with identically distributed. Let, $\{X_n\}_{n \geq 1}$ have a common distribution $G$, $\{Y_n\}_{n\geq 1}$ have a common distribution $H$ and the cycle lengths $\{X_n+Y_n\}_{n \geq 1}$ have a common distribution $F$. Notice that $X_n$ and $Y_m$ are independent for any $n \neq m$, but $X_n$ and $Y_n$ can be dependent for $n\geq 1$. In this setting, interarrival times $\{T_n\}_{n \geq 1}$ can be expressed as:
\[ T_n = \begin{cases} X_{\frac{n+1}{2}} & \text{if } n \text{ is odd}, \\ Y_{\frac{n}{2}} & \text{if } n \text{ is even}. \end{cases} \]

Thus, it can be seen that an alternating renewal process is a type of renewal process. Note that in a renewal process, events, arrivals, and renewals are used interchangeably. However, in alternating renewal processes, a renewal is defined as the alternation between an active phase and an inactive phase. To avoid confusion, we distinguish and use $S_{2n}$ to denote the time of the $n$-th renewal and the $(2n)$-th event and arrival. Furthermore, the focus is more on renewals, thus \( N(t) \) represents the number of renewals in the interval $[0, t]$. 

Define instantaneous availability $A(t)$ as 
\begin{equation*}
A(t) = \mathbb{P}\left(\{\text{system is on at time} \,\,t\}\right).
\end{equation*}
Then, we have the following asymptotic formula for $A(t)$. 
\begin{theorem}\label{thm:limit}
\citep{ross1995stochastic} If \,$\mathbb{E}[X_n + Y_n] < \infty$ and the alternating renewal process is aperiodic then
\begin{equation*}
\lim_{t \to \infty}A(t) = \frac{\mathbb{E}[X_n]}{\mathbb{E}[X_n] +\mathbb{E}[Y_n]}.
\end{equation*}
\end{theorem}

Theorem~\ref{thm:limit} describes the ratio of `on' and `off' states in the limit for a system that follows an alternating renewal process. We called this limit the steady-state availability, which is generally more useful than instantaneous availability for analytical analysis in alternating renewal processes.

\section{Proofs}\label{app:proofs}
For a function $u$, we denote the Laplace transform of $u(t)$ by 
\begin{equation*}
    \mathcal{L}\{u\}(s) = \int_{0}^{\infty} e^{-st} u(t) \, \mathrm{d}t.
\end{equation*}

\begin{proof}[Proof of Theorem~\ref{thm:drop_rate}]
From the definition of the dropout rate in continuous-time, we have $p = 1- A(T)$ for given $p\in(0,1)$ where $A(T):= \sP(\{z^{(i)}(T) \mbox{ is in the active state}\})$\footnote{
As each $z^{(i)}(t)$ has common distributions for $\{X_n^{(i)}\}_{n\geq1}$ and $\{Y_n^{(i)}\}_{n\geq1}$, we have $\sP(\{z^{(1)}(T) \mbox{ is in the active state}\})=\cdots =\sP(\{z^{(d_z)}(T) \mbox{ is in the active state}\})$.
}.

Recall that  $\{X_n^{(i)}\}_{n \geq 1}$ have a common distribution $\text{Exp}(\lambda_1)$, $\{Y_n^{(i)}\}_{n\geq 1}$ have a common distribution $\text{Exp}(\lambda_2)$ and the cycle lengths $\{X_n^{(i)}+Y_n^{(i)}\}_{n \geq 1}$ have a common distribution $F$. Denote the density of $F$ as $f$. Then, using the convolution theorem, the Laplace transform of $f$ is given by
\begin{equation}\label{eq:laplace_f}
\begin{aligned}[b]
    \mathcal{L}\{f\}(s) & =  \int_{0}^{\infty }e^{-st}\left( \lambda_1e^{-\lambda_1t} \right)\,\mathrm{d}t  \int_{0}^{\infty }e^{-st}\left( \lambda_2e^{-\lambda_2t} \right)\,\mathrm{d}t \\
    & = \frac{\lambda_1}{s+\lambda_1}\,\frac{\lambda_2}{s+\lambda_2}.
\end{aligned}
\end{equation}


Moreover, following the computation in \citet{whitt2013IEOR3106}, $A(t)$ can be further expressed by conditioning on the value of $S_2^{(i)} = X_1^{(i)} + Y_1^{(i)}$
\begin{align}
    A(t) &=  \sP(\{z^{(i)}(t) \mbox{ is in the active state}\}) \nonumber \\
    &= \sP(\{z^{(i)}(t) \mbox{ is in the active state}\}, S_2^{(i)} > t)  + \sP(\{z^{(i)}(t) \mbox{ is in the active state}\}, S_2^{(i)} \leq t) \nonumber  \\
    &= \mathbb{P}\left( S_1^{(i)} > t \right) + \int_{0}^{t}\mathbb{P}\left(\{z^{(i)}(t) \mbox{ is in the active state}\} \mid S_2^{(i)} = \tau\right)f(\tau)\,\mathrm{d}\tau  \nonumber \\
    &= \mathbb{P}\left( X_1^{(i)} > t \right) + \int_{0}^{t}\sP(\{z^{(i)}(t-\tau) \mbox{ is in the active state}\})f(\tau)\,\mathrm{d}\tau \nonumber \\
    &= e^{-\lambda_1 t} + \int_{0}^{t}A(t-\tau)f(\tau)\,\mathrm{d}\tau.\label{eq:avail}
\end{align}

Taking the Laplace transform on both sides of \Eqref{eq:avail} yields, 
\begin{equation*}
    \mathcal{L}\{A\}(s) =\frac{1}{s+\lambda_1} + \mathcal{L}\{A\}(s)\,\mathcal{L}\{f\}(s).
\end{equation*}

Thus, we have using \Eqref{eq:laplace_f}
\begin{align*}
    \mathcal{L}\{A\}(s) &= \frac{1}{(1 - \mathcal{L}\{f\}(s))(s+\lambda_1)} \\
    &= \frac{s+\lambda_2}{s(s+\lambda_1+\lambda_2)} \\
    &= \frac{\lambda_2}{\lambda_1+\lambda_2}\,\frac{1}{s}+\frac{\lambda_1}{\lambda_1+\lambda_2}\,\frac{1}{s+\lambda_1+\lambda_2}.
\end{align*}

Then, $A(t)$ is uniquely determined as
\begin{equation*}
    A(t) = \frac{\lambda_2}{\lambda_1+\lambda_2}+\frac{\lambda_1}{\lambda_1+\lambda_2}e^{-(\lambda_1+\lambda_2)t}.
\end{equation*}

Thus, we have the desired result:
\begin{equation*}
\begin{aligned}
    p &= 1 - A(T) \\
    &= \frac{\lambda_1}{\lambda_1+\lambda_2}\left ( 1 - e^{-(\lambda_1 + \lambda_2)T} \right ).
\end{aligned}
\end{equation*}
\end{proof}

\begin{proof}[Proof of Corollary~\ref{cor:drop_rate}]

Let $m^{(i)}(t) = \E[N^{(i)}(t)]$. Then, by the Renewal equation in Theorem~\ref{thm:renewal_eq}, 
\begin{equation}\label{eq:renewal_eq}
    m^{(i)}(t) = F(t) + \int_{0}^{t}f(t-\tau)m^{(i)}(\tau)\,\mathrm{d}\tau.
\end{equation}

By taking the Laplace transform on both sides of \Eqref{eq:renewal_eq}, we obtain 
\begin{equation*}\label{eq:Lm}
    \mathcal{L}\{m^{(i)}\}(s) = \mathcal{L}\{F\}(s) + \mathcal{L}\{f\}(s)\,\mathcal{L}\,\{m^{(i)}\}(s).
\end{equation*}
where $F$ and $f$ are distribution and density functions of $\{X_n^{(i)} + Y_n^{(i)}\}_{n\geq1}$, which yields 
\begin{equation*}\label{eq:ms}
\begin{aligned}[b]
    \mathcal{L}\{m^{(i)}\}(s) &= \frac{\mathcal{L}\{F\}(s)}{1 - \mathcal{L}\{f\}(s)} \\
    &= \frac{\lambda_1\lambda_2}{s^2(s+\lambda_1+\lambda_2)} \\
    &= \frac{\lambda_1\lambda_2}{\lambda_1+\lambda_2}\,\frac{1}{s^2} - \frac{\lambda_1\lambda_2}{(\lambda_1+\lambda_2)^2}\left( \frac{1}{s} - \frac{1}{s+\lambda_1+\lambda_2} \right).
\end{aligned}
\end{equation*}

Therefore, $m^{(i)}(t)$ is given by 
\begin{equation*}
    m^{(i)}(t) = \frac{\lambda_1\lambda_2}{\lambda_1+\lambda_2}\,t - \frac{\lambda_1\lambda_2}{(\lambda_1+\lambda_2)^2}\left( 1-e^{-(\lambda_1+\lambda_2)t} \right).
\end{equation*}

From the definition of $m= \E[N^{(i)}(T)]=m^{(i)}(T)$ and Theorem~\ref{thm:drop_rate}, we have for given $p\in(0,1)$ and $m>0$ 
\begin{align*}
p &= \dfrac{\lambda_1}{\lambda_1+\lambda_2}\left ( 1 - e^{-(\lambda_1 + \lambda_2)T} \right ), \\
m &= \dfrac{\lambda_1\lambda_2}{\lambda_1+\lambda_2}\,T - \dfrac{\lambda_1\lambda_2}{(\lambda_1+\lambda_2)^2}\left( 1-e^{-(\lambda_1+\lambda_2)T} \right).
\end{align*}

Furthermore, from Theorem~\ref{thm:limit}, $\lim_{T\rightarrow \infty} A(T)$ can be easily computed as follows:
\begin{equation*}
\lim_{T \rightarrow \infty}A(T) = \frac{\mathbb{E}[X_n^{(i)}]}{\mathbb{E}[X_n^{(i)}] +\mathbb{E}[Y_n^{(i)}]} = \frac{1/\lambda_1}{1/\lambda_1 + 1/\lambda_2} = \frac{\lambda_2}{\lambda_1 + \lambda_2}.
\end{equation*}

In other words, we have the following asymptotic relation: for sufficiently large $T$,
\begin{align}
p &= 1- A(T)\nonumber \\
&\approx 1- \frac{\lambda_2}{\lambda_1 + \lambda_2}. \label{eq:asymp1}
\end{align}

Furthermore, using Theorem~\ref{thm:element}, we get for sufficiently large $T$
\begin{align}
    \frac{m}{T} &= \frac{m^{(i)}(T)}{T} \nonumber \\
    &\approx \frac{\lambda_1\lambda_2}{\lambda_1+\lambda_2}. \label{eq:asymp2}
\end{align}

From \Eqref{eq:asymp1} and \Eqref{eq:asymp2}, $\lambda_1$ and $\lambda_2$ can be approximated by for large $T$
\begin{equation*}
\lambda_1 \approx \frac{m}{(1-p)\,T}, \quad \lambda_2 \approx \frac{m}{p\,T}.
\end{equation*}

\end{proof}

\section{Description of Datasets}\label{sec:datasets}

\paragraph{SmoothSubspace.}
The SmoothSubspace dataset~\citep{huang2016time} is designed to test algorithms' ability to identify and classify smooth trajectories in a high-dimensional space. This dataset consists of 150 simulated univariate time series, each represented in a 15-dimensional space. The time series are generated such that they lie on or near a smooth subspace within this higher-dimensional space. Each series can belong to one of three classes, with the class indicating the particular subspace configuration that the series aligns with. 

\paragraph{ArticularyWordRecognition.}
The ArticularyWordRecognition dataset~\citep{wang2013word} uses an Electromagnetic Articulograph (EMA) to measure tongue and lip movements during speech. It includes data from multiple native English speakers producing 25 words. The dataset features recordings from twelve sensors capturing $\rvx$, $\rvy$, and $\rvz$ positions at a 200 Hz sampling rate. Sensors are positioned on the forehead, tongue (T1 to T4), lips, and jaw. Out of 36 possible dimensions, this dataset includes 9.

\paragraph{ERing.}
The ERing dataset~\citep{wilhelm2015ering} is collected using a prototype finger ring known as eRing, which detects hand and finger gestures through electric field sensing. The dataset includes six classes of finger postures involving the thumb, index finger, and middle finger. Each data series is four-dimensional with 65 observations, representing measurements from electrodes sensitive to the distance from the hand.

\paragraph{RacketSports.}
The RacketSports dataset consists of data collected from university students playing badminton or squash while wearing a Smart Watch. The watch recorded $\rvx$, $\rvy$, and $\rvz$ coordinates from both the accelerometer and gyroscope, which were transmitted to a phone and stored in an Attribute-Relation File Format (ARFF) file. The dataset includes accelerometer and gyroscope measurements in the order: accelerometer $\rvx$, $\rvy$, $\rvz$, followed by gyroscope $\rvx$, $\rvy$, $\rvz$. The data was collected at 10 Hz over 3 seconds during either forehand/backhand strokes in squash or clear/smash strokes in badminton. 

\paragraph{Speech Commands.}
The Speech Commands dataset~\citep{warden2018speech} features a comprehensive set of one-second long audio clips that include spoken words and ambient sounds. This dataset contains 34,975 time-series entries, each corresponding to one of 35 different words. The audio clips are pre-processed by computing the Mel-frequency cepstral coefficients, which serve as features to capture the audio characteristics more effectively, thereby enhancing the accuracy of the applied machine learning models. Each audio sample is represented by a time series of 161 frames, with each frame consisting of 20 feature dimensions.

\paragraph{PhysioNet Sepsis.}
The 2019 PhysioNet Sepsis challenge~\citep{reyna2019early} focuses on predicting sepsis, a critical condition caused by blood infections leading to numerous fatalities. This challenge uses a dataset with 40,335 ICU patient records, featuring 34 time-dependent indicators like heart rate and body temperature. Researchers aim to determine the presence of sepsis in patients based on the sepsis-3 criteria. 
This dataset presents challenges due to its irregular time-series nature—only 10\% of data points are timestamped for each patient. To tackle this, we employ two strategies for time-series classification: (i) using observation intensity (OI), which gauges the severity of the patient's condition, and (ii) without using OI. 

\paragraph{CIFAR-100 and CIFAR-10.}
These datasets, introduced by \citet{krizhevsky2009learning}, contain 60,000 32x32 pixel RGB images that are split into two groups: CIFAR-100, with 100 different fine-grained classes, and CIFAR-10, with 10 broader categories. Each dataset is evenly divided, with 50,000 images allocated for training and 10,000 for testing. The images depict a wide variety of objects, including animals, vehicles, and everyday scenes.

\paragraph{STL-10.} 
The STL-10 dataset~\citep{coates2011analysis} consists of 13,000 color images, each with a resolution of 96x96 pixels, categorized into 10 classes. The images primarily capture animals, objects, and outdoor scenes.

\paragraph{SVHN.}
The Street View House Numbers (SVHN) dataset~\citep{netzer2011reading} consists of over 600,000 digit images extracted from real-world scenes, specifically house numbers captured by Google's Street View. The images typically have a resolution of around 32x32 pixels and contain digits from `0' to `9'. Each image often includes multiple digits, but the dataset focuses on identifying individual digits.





\section{Details of Experimental Settings}\label{sec:settings}
We followed the recommended pipelines for each dataset and model. We conducted experiments repeated for five iterations to ensure robustness in all settings. For time series classification experiments, we adhered to experimental protocols provided in the GitHub repositories of each model. For image classification experiments, as there were no publicly available protocols for NDEs, we established our own. Our experiments were performed using a server on Ubuntu 20.04.6 LTS, equipped with an Intel(R) Core(TM) i9-10980XE CPU and four NVIDIA GeForce RTX 4090 GPUs. The source code can be accessed at \url{https://github.com/jonghun-lee0/Continuum-Dropout}.

\subsection{Various Regularization Methods}

In `Na\"{\i}ve Dropout for Drift Network' and 'Na\"{\i}ve Dropout for MLP Classifier'~\citep{srivastava2014dropout}, we used \( p = [0.1, 0.2, 0.3, 0.4, 0.5] \) for tuning the dropout rate. For `Jump Diffusion model'~\citep{liu2020does}, since they did not provide a reproducible code implementation, we referred to their Github repository\footnote{{\url{https://github.com/xuanqing94/NeuralSDE}}} to devise our own pipeline for the experiments. We employed tuning methods for typical dropout rates; however, since we could not observe significant performance improvements, we utilized a broader tuning grid of \( p = [10^{-5},  10^{-4}, 10^{-3}, 10^{-2}, 0.1, 0.2, 0.3, 0.4, 0.5] \). For `STEER', we followed the pipeline outlined by \citet{ghosh2020steer} and the GitHub repository\footnote{{\url{https://github.com/arnabgho/steer}}}. For `Temporal Adaptive Batch Normalization (TA-BN)', we followed the pipeline outlined by \citet{zheng2024improving} and the GitHub repository\footnote{{\url{https://github.com/shelljane/tabn-neuralode}}}. For further details, please refer to the original papers. 


 
\subsection{Time Series Classification}

In `SmoothSubspace', `ArticularyWordRecognition', `ERing' and `RacketSports', we adhered to the experimental protocol using the pipeline outlined by \citet{oh2024stable} and Github repository\footnote{\url{https://github.com/yongkyung-oh/Stable-Neural-SDEs}\label{git:oh}}. We utilized a 70\%:15\%:15\% split for train, validation, and test due to the unconventional split ratios in the original datasets, as recommended by \citet{oh2024stable}. Regarding hyperparameter tuning, we employed the Python library \texttt{ray}\footnote{\url{https://github.com/ray-project/ray}}~\citep{moritz2018ray,liaw2018tune}, as suggested by \citet{oh2024stable}. For further details, please refer to the original paper.

In `Speech Commands' and `PhysioNet Sepsis', we followed the experimental protocol using publicly available pipelines for each model. For ANCDE, we used the pipeline outlined by \citet{jhin2023attentive} and the Github repository\footnote{\url{https://github.com/sheoyon-jhin/ANCDE}}. For the other models, we followed the pipeline outlined by \citet{oh2024stable} and the Github repository\footref{git:oh}. However, for the ANCDE model using the `Speech Commands' dataset, the hyperparameter settings for the architecture were not disclosed. Therefore, we conducted experiments using the number of layers $n_l = 4$, the hidden vector dimensions $n_h = 128$, and the hidden vector dimensions for attention $n_{attention} = 20$, and we reported the performance based on these settings. We performed Continuum Dropout hyperparameter tuning for two architecture (optimal and complex) for Neural CDE, Neural LSDE, Neural LNSDE, and Neural GSDE. For Neural CDE, contrary to previous approaches, we utilized grid search to optimize the model hyperparameters because the prior settings failed to replicate the reported results. We evaluated test performance using $n_l$ from the set $\{1, 2, 3, 4\}$ and $n_h$ from $\{16, 32, 64, 128\}$, selecting both the \textbf{optimal} architecture and the most \textbf{complex} architecture. Note that in certain models, optimal hyperparameters might be the same as complex hyperparameters. For Neural LSDE, Neural LNSDE, and Neural GSDE, we used the optimal architecture hyperparameters reported by \citet{oh2024stable}.
The following are the definitions of each model:


\paragraph{GRU-ODE.}
\citet{de2019gru} introduced GRU-ODE, a concept of combination of Neural ODE and GRU as the solution of the following ordinary differential equation:
\[ \frac{\mathrm{d}\mathbf{z}_0(t)}{\mathrm{d}t} = (1 - \mathbf{u}_{gate}(t)) \circ (\mathbf{u}_{vector}(t) - \mathbf{z}_0(t)). \]
Here, $\mathbf{u}_{gate}(t)$ and $\mathbf{u}_{vector}(t)$ represent the update gate and update vector of the GRU, respectively. They control how much of the new state information should be retained and how much of the previous state should be forgotten.\\
The reset gate $\rvu_{reset}(t)$ and update gate $\mathbf{u}_{gate}(t)$ are defined as:
\begin{align*}
\mathbf{u}_{reset}(t) &= \mathrm{sigmoid}(\mathbf{W}_{r} \mathbf{x}(t) + \mathbf{V}_{r} \mathbf{z}_0(t) + \mathbf{b}_{r}), \\
\mathbf{u}_{gate}(t) &= \mathrm{sigmoid}(\mathbf{W}_{g} \mathbf{x}(t) + \mathbf{V}_{g} \mathbf{z}_0(t) + \mathbf{b}_{g}),
\end{align*}
where $\mathbf{W}_{r}$, $\mathbf{V}_{r}$, $\mathbf{W}_{g}$ and $\mathbf{V}_{g}$ are weight matrices, $\mathbf{b}_{r}$ and $\mathbf{b}_{g}$ are bias vectors, and $\mathbf{x}(t)$ is the input data at time \( t \).

The update vector $\mathbf{u}_{vector}(t)$ is given by:
\[ \mathbf{u}_{vector}(t) = \tanh(\mathbf{W}_{v} \mathbf{x}(t) + \mathbf{V}_{v} (\mathbf{u}_{reset}(t) \circ \mathbf{z}_0(t)) + \mathbf{b}_{v}), \]
where $\mathbf{W}_{v}$ and $\mathbf{V}_{v}$ are weight matrices, and $\mathbf{b}_{v}$ is a bias vector.

\paragraph{ODE-RNN.}
The ODE-RNN, proposed by \citet{rubanova2019latent}, combines Neural ODEs and RNNs. The latent process $\mathbf{z}_0(t)$ is obtained by numerically solving the ODE and then undergoing a standard RNN update process:
\begin{align*}
\bar{\mathbf{z}}_0(t_i) &= \text{ODE\_Solver}(\gamma(\cdot; \cdot; \theta_{\gamma}), \mathbf{z}_0(t_{i-1}), [t_{i-1}, t_i]), \\
\mathbf{z}_0(t_i) &= \text{RNNCell}(\bar{\mathbf{z}}_0(t_i), x^{(i)}) \quad \text{with }\, \mathbf{z}_0(t_0) = 0,
\end{align*}
where the drift function \(\gamma(\cdot; \cdot; \theta_{\gamma})\) is a neural network with parameter \(\theta_{\gamma}\).

\paragraph{Neural CDE.}
\citet{kidger2020cde} proposed Neural Controlled Differential Equation (Neural CDE), an extension of RNN to a continuous-time setting. It addresses the limitation of the initial condition determining the solution in the existing Neural ODE by introducing the concept of a controlled path $\mathbf X(t)$, which incorporates data arriving later. The model is formulated as follows:
\begin{equation*}
\mathbf{z}_0(t) = \mathbf{z}_0(0) + \int_{0}^{t}\gamma(s, \mathbf{z}_0(s);\theta_{\gamma})\,\mathrm{d}\mathbf X(s)\quad\text{with}\,\,\mathbf{z}_0(0)=\zeta(x^{(1)};\theta_{\zeta}),
\end{equation*}
where the drift function \(\gamma(\cdot; \cdot; \theta_{\gamma})\) is a neural network with parameter \(\theta_{\gamma}\) and the integral is the Riemann-Stieltjes integral. A controlled path $\mathbf X(t)$ can be any continuous function of bounded variation, but we have chosen natural cubic spline of $\mathbf{x}$ in our experiments.

\paragraph{ANCDE.}
\citet{jhin2023attentive} proposed the Attentive Neural Controlled Differential Equation (ANCDE), where they adopt two Neural CDEs: the bottom Neural CDE for attention $\mathbf{z}_{attention}(t)$ and the top Neural CDE for latent process $\mathbf{z}_0(t)$:
\begin{align*}
\mathbf{z}_{attention}(t) &= \mathbf{z}_{attention}(0) + \int_{0}^{t}\gamma_{1}(s, \mathbf{z}_{attention}(s);\theta_{\gamma_{1}})\,\mathrm{d}\mathbf X(s), \\
\mathbf{z}_0(t) &= \mathbf{z}_0(0) + \int_{0}^{t}\gamma_{2}(s, \mathbf{z}_0(s);\theta_{\gamma_{2}})\,\mathrm{d}\mathbf Y(s),
\end{align*}
where $\mathbf Y(t) = \mathrm{sigmoid}(\mathbf{z}_{attention}(t)) \circ \mathbf X(t)$ and the drift functions \(\gamma_1(\cdot; \cdot; \theta_{\gamma_1})\) and \(\gamma_2(\cdot; \cdot; \theta_{\gamma_2})\) are neural networks with parameter \(\theta_{\gamma_1}\) and \(\theta_{\gamma_2}\), respectively.


\citet{oh2024stable} proposed three classes of Neural SDEs: Neural Langevin-type SDE (LSDE), Neural Linear Noise SDE (LNSDE), and Neural Geometric SDE (GSDE). These models incorporate controlled paths into well-established SDEs, effectively capturing sequential observations like time series data and achieving recent state-of-the-art performance.

\paragraph{Neural LSDE.}
Neural LSDE is a class of Langevin SDE, defined as follows:
\begin{equation*}
\mathbf{z}_0(t) = \mathbf{z}_0(0) + \int_{0}^{t}\gamma(\mathbf{z}_0(s);\theta_{\gamma})\,\mathrm{d}s + \int_{0}^{t}\sigma(s;\theta_{\sigma})\,\mathrm{d}\mathbf{W}(s) \quad\text{with}\,\,\mathbf{z}_0(0)=\zeta(\mathbf{x};\theta_{\zeta}),
\end{equation*}
where the drift function \(\gamma(\cdot; \theta_{\gamma})\) is a neural network with parameter \(\theta_{\gamma}\) and the diffusion function \(\sigma(\cdot; \theta_{\sigma})\) is a neural network with parameter \(\theta_{\sigma}\).

\paragraph{Neural LNSDE.}
Neural LNSDE is an SDE with linear multiplicative noise, defined as follows:
\begin{equation*}
\mathbf{z}_0(t) = \mathbf{z}_0(0) + \int_{0}^{t}\gamma(s, \mathbf{z}_0(s);\theta_{\gamma})\,\mathrm{d}s + \int_{0}^{t}\sigma(s;\theta_{\sigma})\mathbf{z}_0(s)\,\mathrm{d}\mathbf{W}(s) \quad\text{with}\,\,\mathbf{z}_0(0)=\zeta(\mathbf{x};\theta_{\zeta}),
\end{equation*}
where the drift function \(\gamma(\cdot; \cdot; \theta_{\gamma})\) is a neural network with parameter \(\theta_{\gamma}\) and the diffusion function \(\sigma(\cdot; \theta_{\sigma})\) is a neural network with parameter \(\theta_{\sigma}\).

\paragraph{Neural GSDE.}
Neural GSDE is motivated by Geometric Brownian motion (GBM) and is defined as follows:
\begin{equation*}
\mathbf{z}_0(t) = \mathbf{z}_0(0) + \int_{0}^{t}\gamma(s, \mathbf{z}_0(s);\theta_{\gamma})\mathbf{z}_0(s)\,\mathrm{d}s + \int_{0}^{t}\sigma(s;\theta_{\sigma})\mathbf{z}_0(s)\,\mathrm{d}\mathbf{W}(s) \quad\text{with}\,\,\mathbf{z}_0(0)=\zeta(\mathbf{x};\theta_{\zeta}),
\end{equation*}
where the drift function \(\gamma(\cdot; \cdot; \theta_{\gamma})\) is a neural network with parameter \(\theta_{\gamma}\) and the diffusion function \(\sigma(\cdot; \theta_{\sigma})\) is a neural network with parameter \(\theta_{\sigma}\).

\newpage
\subsection{Image Classification}\label{sec:image_experi}

For the dataset preprocessing, we adhered to the original train-test split provided with the datasets. From the training set, 20\% was reserved for validation purposes. For data augmentation in the training set, we employed random resizing, cropping, and flipping techniques. Image normalization was carried out using the original mean and standard deviation values from each dataset, and the images were used at their default sizes without any resizing. This approach helped maintain the integrity of the original data while enhancing the model's ability to generalize from augmented variations.

In our experiments, the batch size was set to 128 for the CIFAR-100, CIFAR-10, and SVHN datasets. For the STL-10 dataset, a smaller batch size of 64 was utilized due to its limited number of samples. The models were optimized using stochastic gradient descent (SGD) with an initial learning rate of 0.1 over a course of 100 epochs. To address potential stalls in training progress, the learning rate was halved if there was no change in the loss over two consecutive epochs. Additionally, an early-stopping mechanism was implemented, terminating the training process if there was no improvement in the loss for ten consecutive epochs. This strategy helps in preventing overfitting and ensures efficient training by curtailing unnecessary computation once performance plateaus.

Specifically, the model architecture consists of the following components:
\begin{itemize}
    \item The input layer of the image classification model is similar to conventional CNN models, where multiple convolutional operations are applied to the input image to extract basic features. An average pooling layer is used to extract the input feature as a vector $\rvz$, with a vector size of 1024 obtained through convolutional layers.
    \item The feature vector $\rvz$ is fed into the neural differential equation module, which estimates the hidden state $\rvz(T)$ from the initial input feature $\rvz$, where $T$ represents the depth. The last value of $\rvz(T)$ is estimated by solving neural differential equations.
    \item The differential equation module consists of a vector field that can be learned through backpropagation. In the case of Neural SDE, the drift term and diffusion term are represented by neural networks. The complexity of the vector field is controlled to ensure that the total number of parameters is similar to ResNet18.
    \item The last value $\rvz(T)$ obtained from the neural differential equation module is fed into the classifier to identify the class of the input. The parameters for all three modules (feature extractor, neural differential equation module, and classifier) are optimized simultaneously using backpropagation and the adjoint-sensitive algorithm. 
\end{itemize}

We defined the complexity of the model to have a number of parameters comparable to that of ResNet18. Precisely, ResNet18 has 11.18 million parameters for 10 classes and 11.23 million for 100 classes. Our Neural ODE models possess 10.49 million and 10.59 million parameters for 10 and 100 classes, respectively. Meanwhile, the Neural SDE models, both additive and multiplicative, contain 12.59 million and 12.69 million parameters for 10 and 100 classes, respectively. 
We used depth $T=10$ in the experiments. We employed the Euler method (Euler-Maruyama method in the case of SDEs) for the numerical solving of the differential equations, using a time step size of $\Delta t = 0.1$.

Neural SDE with additive and Neural SDE with multiplicative noise are defined as follows:

\paragraph{Neural SDE with Additive Noise.}

\begin{equation*}
\mathbf{z}_0(t) = \mathbf{z}_0(0) +\int_{0}^{t} \gamma(s, \mathbf{z}_0(s);\theta_{\gamma})\,\mathrm{d}s + \int_{0}^{t}\sigma(s;\theta_{\sigma})\,\mathrm{d} \mathbf{W}(s) \quad\text{with}\,\,\mathbf{z}_0(0)=\zeta(\mathbf{x};\theta_{\zeta}),
\end{equation*}
where the noise term does not depend on the latent process \(\mathbf{z}_0(t)\) and only depends on time $t$.

\paragraph{Neural SDE with Multiplicative Noise.}
\begin{equation*}
\mathbf{z}_0(t) = \mathbf{z}_0(0) +\int_{0}^{t} \gamma(s, \mathbf{z}_0(s);\theta_{\gamma})\,\mathrm{d}s + \int_{0}^{t}\sigma(s;\theta_{\sigma})\,\mathbf{z}_0(s)\,\mathrm{d} \mathbf{W}(s) \quad\text{with}\,\,\mathbf{z}_0(0)=\zeta(\mathbf{x};\theta_{\zeta}),
\end{equation*}
where the noise term indicates an interaction between the diffusion function and the latent process \(\mathbf{z}_0(t)\).

\subsection{Calibration Experiments}\label{sec:app_cali}
We designed experiments to evaluate the uncertainty quantification capability of the proposed method (Continuum Dropout), particularly its ability to produce uncertainty-aware predictions. We assessed calibration performance on CIFAR-100 with Neural ODE and Speech Commands with Neural CDE—the same settings in which the proposed method had previously demonstrated strong regularization performance. Importantly, we did not tune hyperparameters specifically for calibration; instead, we reused the hyperparameters that achieved the best classification performance for each dataset and model. Predictive confidence was computed from the softmax probabilities, and calibration was evaluated using reliability diagrams with 10 equal-width bins. For Continuum Dropout, we performed 5 Monte Carlo (MC) samples of the latent trajectory $\mathbf{z}(t)$ to estimate 
its stochastic distribution, and applied the classification layer to these samples to compute the final predictive distribution.

\newpage

\section{Extended experiments on time series datasets}

We conducted a comprehensive evaluation of the proposed method (Continuum Dropout) using 30 diverse datasets from the University of East Anglia (UEA) and University of California Riverside (UCR) Time Series Classification Repository\footnote{\url{http://www.timeseriesclassification.com/}}~\citep{bagnall2018uea,dau2019ucr}. This analysis was facilitated by the Python library \texttt{sktime}\footnote{\url{https://github.com/sktime/sktime}}~\citep{loning2019sktime}. 
We followed the experiment protocol and code repository based on \citet{oh2024stable}.
\vspace{-5pt}
\begin{table}[htb]
\scriptsize\centering\captionsetup{justification=centering, skip=5pt}
\caption{Data description for extended experiments}\label{tab:data}
\begin{tabular}{@{}lrrrr@{}}
\Xhline{0.7pt}
\multicolumn{1}{c}{\textbf{Dataset}} & \multicolumn{1}{c}{\textbf{Total number of samples}} & \multicolumn{1}{c}{\textbf{Number of classes}} & \multicolumn{1}{c}{\textbf{Dimension of time series}} & \multicolumn{1}{c}{\textbf{Length of time series}} \\ \hline
\textbf{ArrowHead}                 & 211                           & 3                              & 1                            & 251                     \\
\textbf{Car}                       & 120                           & 4                              & 1                            & 577                     \\
\textbf{Coffee}                    & 56                            & 2                              & 1                            & 286                     \\
\textbf{GunPoint}                  & 200                           & 2                              & 1                            & 150                     \\
\textbf{Herring}                   & 128                           & 2                              & 1                            & 512                     \\
\textbf{Lightning2}                & 121                           & 2                              & 1                            & 637                     \\
\textbf{Lightning7}                & 143                           & 7                              & 1                            & 319                     \\
\textbf{Meat}                      & 120                           & 3                              & 1                            & 448                     \\
\textbf{OliveOil}                  & 60                            & 4                              & 1                            & 570                     \\
\textbf{Rock}                      & 70                            & 4                              & 1                            & 2844                    \\
\textbf{SmoothSubspace}            & 300                           & 3                              & 1                            & 15                      \\
\textbf{ToeSegmentation1}          & 268                           & 2                              & 1                            & 277                     \\
\textbf{ToeSegmentation2}          & 166                           & 2                              & 1                            & 343                     \\
\textbf{Trace}                     & 200                           & 4                              & 1                            & 275                     \\
\textbf{Wine}                      & 111                           & 2                              & 1                            & 234                     \\ \hline
\textbf{ArticularyWordRecognition} & 575                           & 25                             & 9                            & 144                     \\
\textbf{BasicMotions}              & 80                            & 4                              & 6                            & 100                     \\
\textbf{CharacterTrajectories}     & 2858                          & 20                             & 3                            & 60-180                  \\
\textbf{Cricket}                   & 180                           & 12                             & 6                            & 1197                    \\
\textbf{Epilepsy}                  & 275                           & 4                              & 3                            & 206                     \\
\textbf{ERing}                     & 300                           & 6                              & 4                            & 65                      \\
\textbf{EthanolConcentration}      & 524                           & 4                              & 3                            & 1751                    \\
\textbf{EyesOpenShut}              & 98                            & 2                              & 14                           & 128                     \\
\textbf{FingerMovements}           & 416                           & 2                              & 28                           & 50                      \\
\textbf{Handwriting}               & 1000                          & 26                             & 3                            & 152                     \\
\textbf{JapaneseVowels}            & 640                           & 9                              & 12                           & 7-26                    \\
\textbf{Libras}                    & 360                           & 15                             & 2                            & 45                      \\
\textbf{NATOPS}                    & 360                           & 6                              & 24                           & 51                      \\
\textbf{RacketSports}              & 303                           & 4                              & 6                            & 30                      \\
\textbf{SpokenArabicDigits}        & 8798                          & 10                             & 13                           & 4-93                    \\ \Xhline{0.7pt}
\end{tabular}
\vspace{-5pt}
\end{table}

Table~\ref{tab:summary} presents a comprehensive overview of the performance improvements achieved by our proposed method (Continuum Dropout) across various methods. The results demonstrate significant enhancements in classification accuracy across the 30 datasets examined. 
Notably, the impact was more pronounced for Neural ODE and Neural CDE models compared to Neural SDE-based approaches. 
\vspace{-5pt}
\begin{table}[htb]
\scriptsize\centering\captionsetup{justification=centering, skip=5pt}
\caption{Comprehensive performance analysis on extended datasets \\
(Results are averaged across 30 diverse datasets. Values in parentheses represent the mean of individual standard deviations. Improvement and percentage change are reported as average $\pm$ standard error of the mean.)
}\label{tab:summary}
\begin{tabular}{@{}ccccc@{}}
\Xhline{0.7pt}
\textbf{Method}       & \textbf{Baseline} & \textbf{Proposed} & \textbf{Improvement} & \textbf{\% Change} \\ \hline
\textbf{Neural ODE}   & 0.521 (0.065)     & 0.568 (0.054)     & 0.047  ± 0.006       & 11.56  ± 2.02      \\ 
\textbf{Neural CDE}   & 0.709 (0.061)     & 0.781 (0.048)     & 0.072  ± 0.012       & 13.39  ± 2.70      \\ 
\textbf{Neural SDE}   & 0.526 (0.068)     & 0.571 (0.065)     & 0.045  ± 0.008       & 9.69  ± 1.87       \\ 
\textbf{Neural LSDE}  & 0.717 (0.056)     & 0.741 (0.061)     & 0.024  ± 0.010       & 4.19  ± 1.65       \\ 
\textbf{Neural LNSDE} & 0.727 (0.047)     & 0.761 (0.056)     & 0.034  ± 0.007       & 5.93  ± 1.38       \\ 
\textbf{Neural GSDE}  & 0.716 (0.065)     & 0.752 (0.063)     & 0.036  ± 0.009       & 6.00  ± 1.50       \\ \Xhline{0.7pt}
\end{tabular}
\vspace{-5pt}
\end{table}

Tables~\ref{tab:summary_univariate} and~\ref{tab:summary_multivariate} offer a detailed decomposition of our experimental outcomes, stratifying the 30 datasets into two equal subsets: 15 univariate and 15 multivariate time series. Our proposed method demonstrates heightened efficacy when applied to univariate datasets. The contrasting results between these two categories underscore the importance of considering data dimensionality when applying proposed method to NDEs.

\begin{table}[htb]
\scriptsize\centering\captionsetup{justification=centering, skip=5pt}
\caption{Comprehensive performance analysis on 15 univariate datasets
}\label{tab:summary_univariate}
\begin{tabular}{@{}ccccc@{}}
\Xhline{0.7pt}
\textbf{Method}       & \textbf{Baseline} & \textbf{Proposed} & \textbf{Improvement} & \textbf{\% Change} \\ \hline
\textbf{Neural ODE}   & 0.535 (0.073)     & 0.593 (0.068)     & 0.058  ± 0.009       & 11.73  ± 1.93      \\ 
\textbf{Neural CDE}   & 0.628 (0.079)     & 0.725 (0.065)     & 0.098  ± 0.020       & 19.33  ± 4.63      \\ 
\textbf{Neural SDE}   & 0.535 (0.086)     & 0.597 (0.082)     & 0.062  ± 0.013       & 12.99  ± 3.23      \\ 
\textbf{Neural LSDE}  & 0.636 (0.082)     & 0.683 (0.089)     & 0.047  ± 0.013       & 8.12  ± 2.35       \\ 
\textbf{Neural LNSDE} & 0.665 (0.065)     & 0.718 (0.078)     & 0.053  ± 0.010       & 9.41  ± 2.19       \\ 
\textbf{Neural GSDE}  & 0.649 (0.097)     & 0.711 (0.096)     & 0.062  ± 0.015       & 10.30  ± 2.50      \\ \Xhline{0.7pt}
\end{tabular}
\vspace{-3pt}
\end{table}

\begin{table}[htb]
\scriptsize\centering\captionsetup{justification=centering, skip=5pt}
\caption{Comprehensive performance analysis on 15 multivariate datasets
}\label{tab:summary_multivariate}
\begin{tabular}{@{}ccccc@{}}
\Xhline{0.7pt}
\textbf{Method}       & \textbf{Baseline} & \textbf{Proposed} & \textbf{Improvement} & \textbf{\% Change} \\ \hline
\textbf{Neural ODE}   & 0.507 (0.056)     & 0.544 (0.041)     & 0.037  ± 0.009       & 11.38  ± 3.63      \\ 
\textbf{Neural CDE}   & 0.790 (0.044)     & 0.837 (0.031)     & 0.047  ± 0.010       & 7.45  ± 1.93       \\ 
\textbf{Neural SDE}   & 0.516 (0.049)     & 0.545 (0.048)     & 0.028  ± 0.007       & 6.38  ± 1.57       \\ 
\textbf{Neural LSDE}  & 0.798 (0.031)     & 0.800 (0.032)     & 0.002  ± 0.013       & 0.27  ± 1.88       \\ 
\textbf{Neural LNSDE} & 0.789 (0.030)     & 0.805 (0.034)     & 0.016  ± 0.007       & 2.45  ± 1.18       \\ 
\textbf{Neural GSDE}  & 0.783 (0.034)     & 0.794 (0.030)     & 0.011  ± 0.005       & 1.71  ± 0.69       \\ \Xhline{0.7pt}
\end{tabular}
\vspace{-3pt}
\end{table}

Figure~\ref{fig:hyperparam_extended} illustrates the performance variations across different hyperparameter configurations. The baseline performance without Continuum Dropout is represented by a black dashed horizontal line. Our results consistently demonstrate that the proposed method enhances classification performance relative to this baseline across various hyperparameter settings. This robust improvement highlights the method's effectiveness in bolstering generalization capabilities and its versatility in addressing diverse time series classification tasks, further validating its practical utility in real-world problems.

\begin{figure*}[!htb]
    \centering\captionsetup{justification=centering, skip=5pt}
    \captionsetup[subfigure]{justification=centering, skip=5pt}
    \subfloat[Neural ODE]{
      \includegraphics[clip,width=0.47\linewidth]{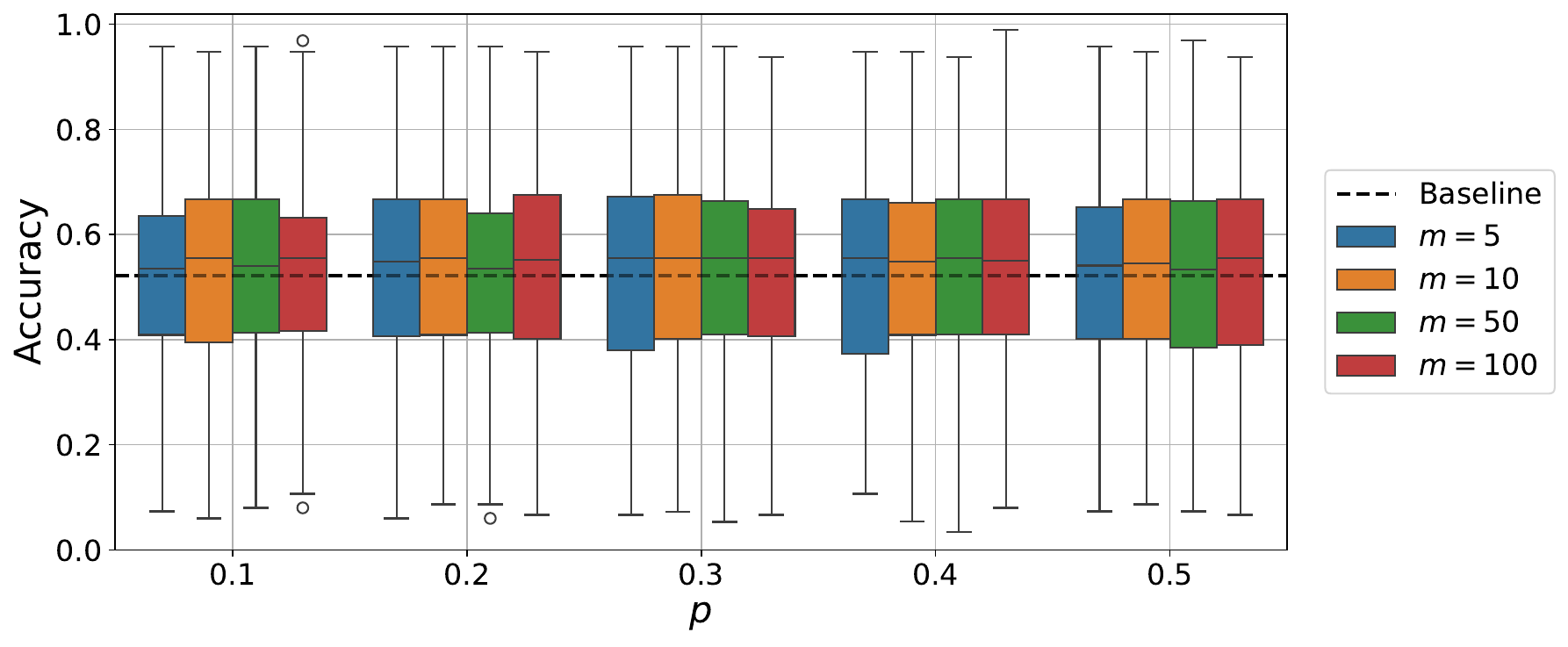}} \hfill
    \subfloat[Neural CDE]{
      \includegraphics[clip,width=0.47\linewidth]{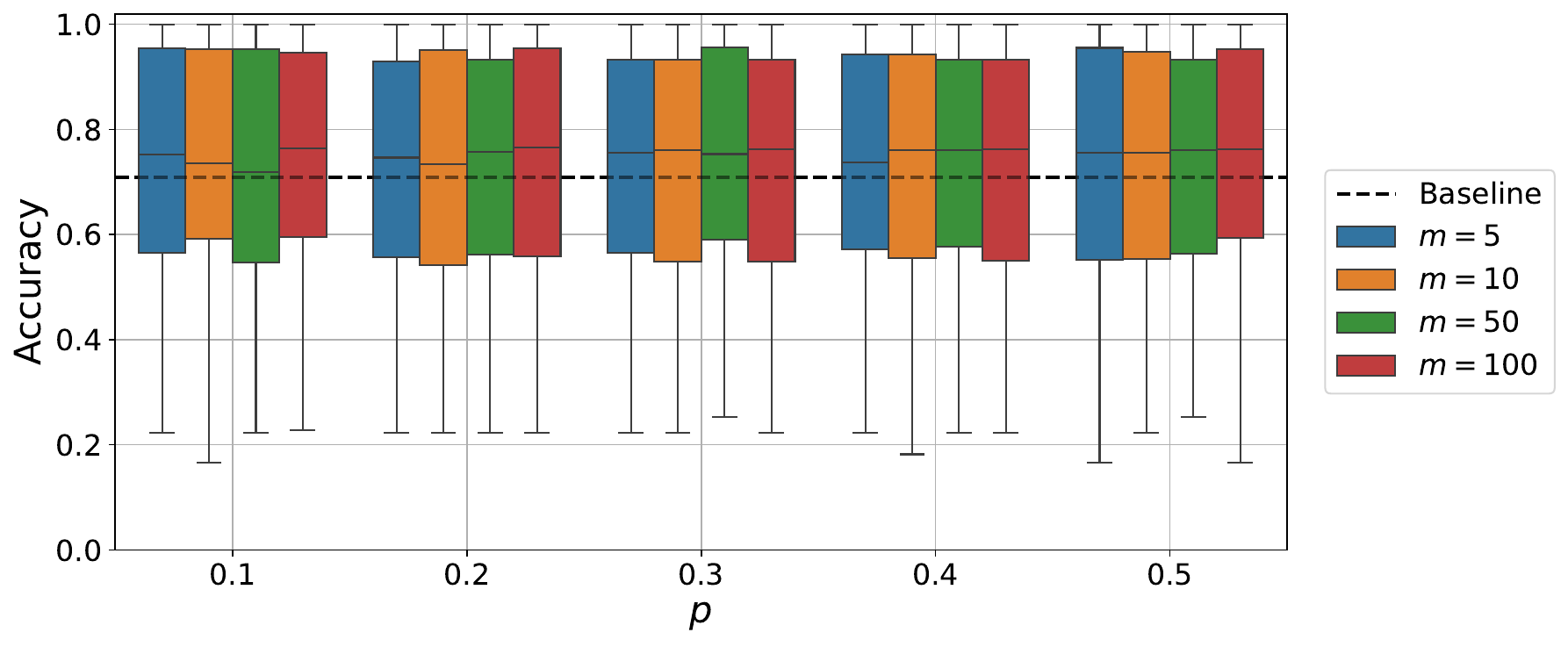}} \\
    \subfloat[Neural SDE]{
      \includegraphics[clip,width=0.47\linewidth]{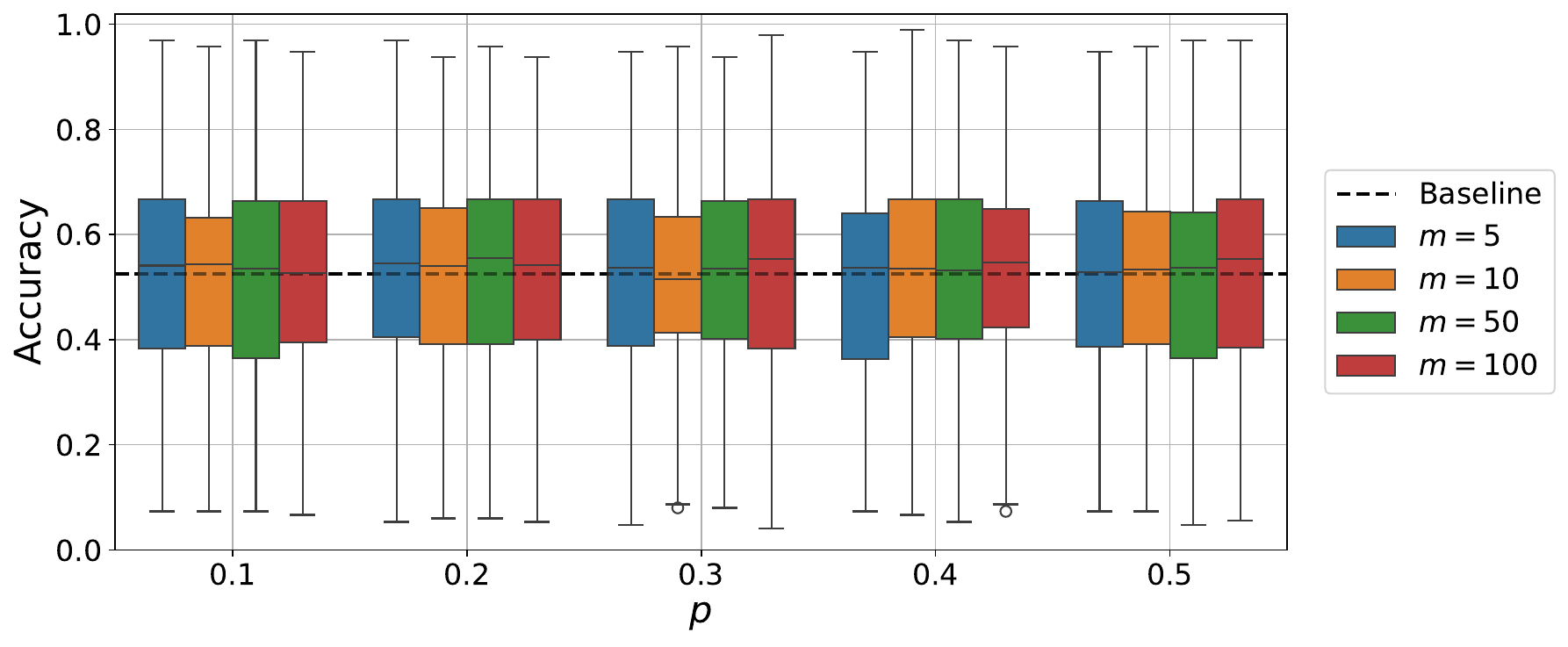}} \hfill
    \subfloat[Neural LSDE]{
      \includegraphics[clip,width=0.47\linewidth]{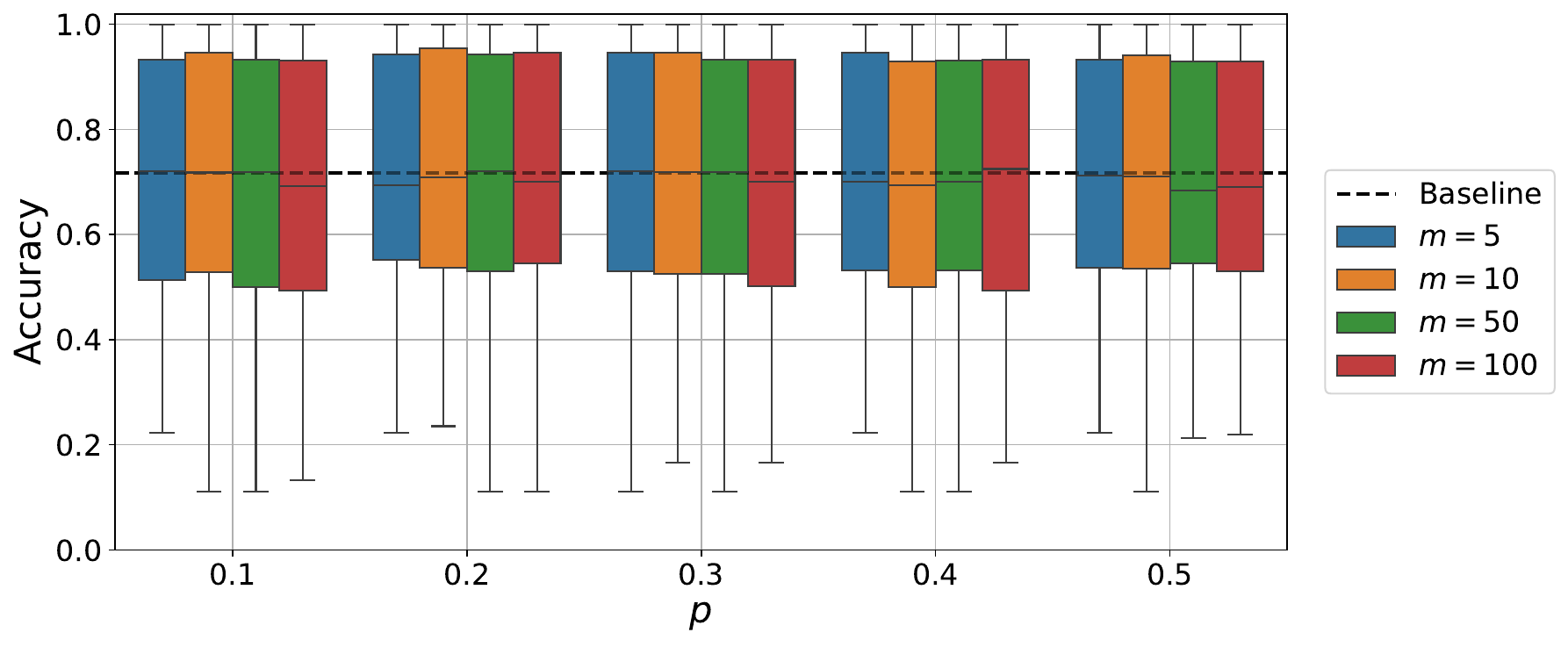}} \\
    \subfloat[Neural LNSDE]{
      \includegraphics[clip,width=0.47\linewidth]{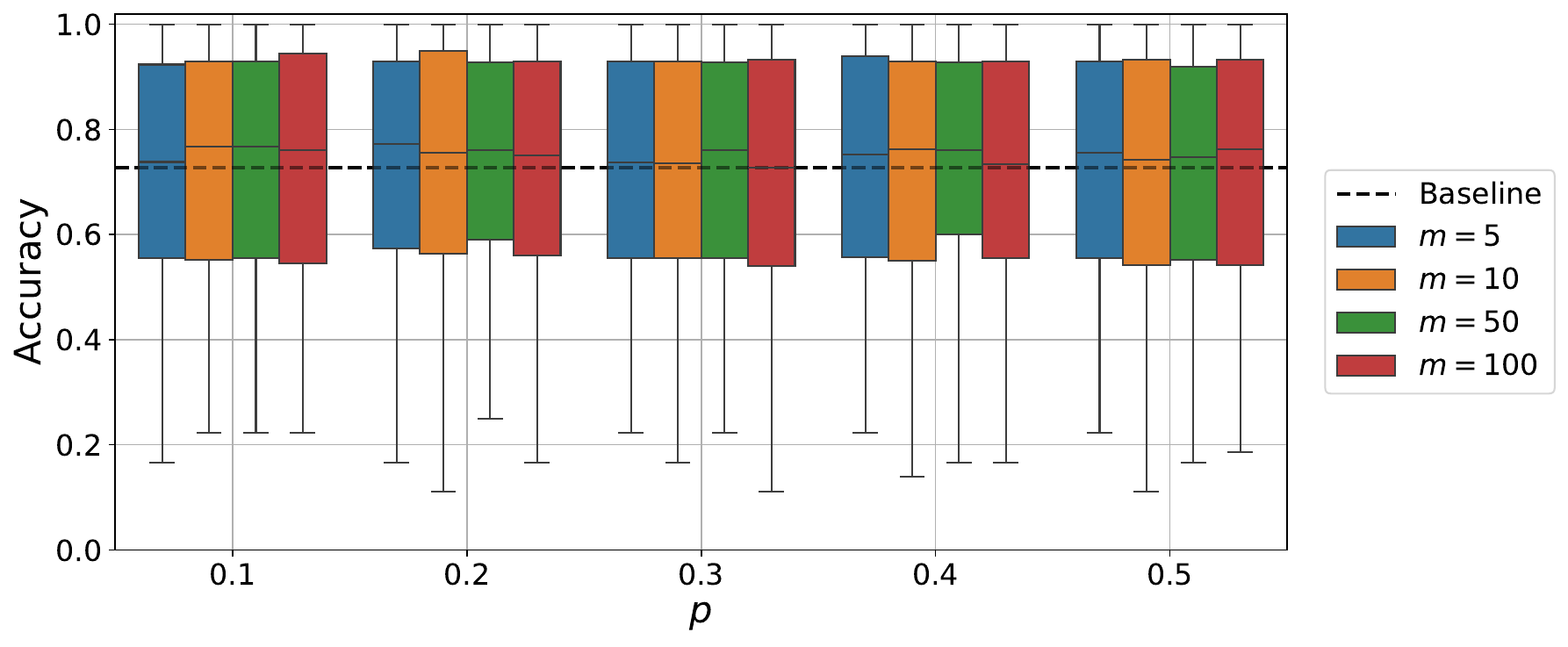}} \hfill
    \subfloat[Neural GSDE]{
      \includegraphics[clip,width=0.47\linewidth]{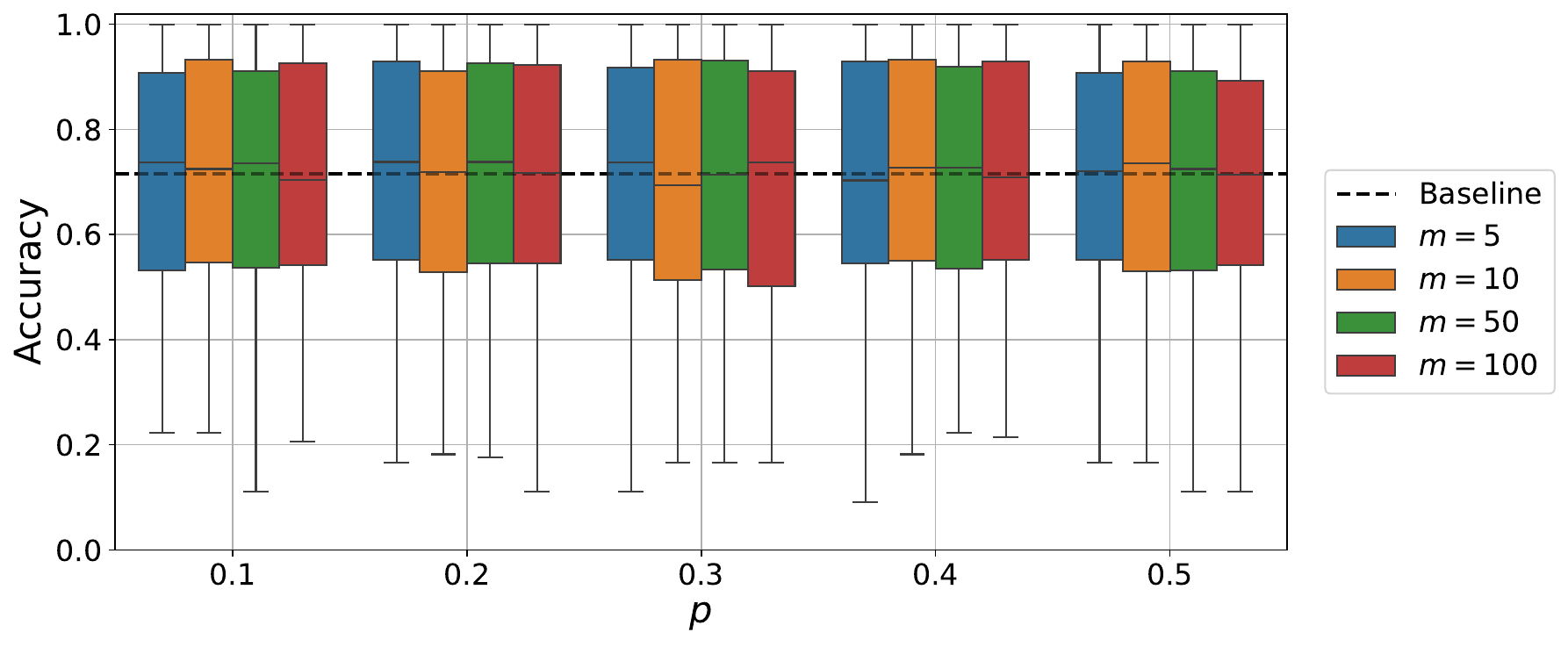}} 
    \caption{Performance comparison with different hyperparameters on extended datasets}\label{fig:hyperparam_extended}
\end{figure*}

\clearpage
\newpage
\section{Further Analysis of Continuum Dropout}\label{app:NRODE}

\subsection{Sensitivity Analysis of Hyperparameters}\label{sec:sensitivity}

We conducted a sensitivity analysis to provide guidelines for selecting the hyperparameters $p$ and $m$ of proposed method (Continuum Dropout). We evaluated the accuracy of Neural CDE successfully trained on Speech Commands dataset. To observe the performance variations with respect to $p$ and $m$, we set $p \in \lbrack 0.1, 0.2, 0.3, 0.4, 0.5 \rbrack$ and $m \in \lbrack 3, 5, 10, 50, 100 \rbrack$. Table~\ref{tab:sensitivity} shows the performance for each combination of $p$ and $m$, and we highlighted in red the cases that exhibit lower performance than the performance without Continuum Dropout, adjusted by one standard deviation. We observed a significant performance decline at high dropout rates when $m=3$. This is because a smaller $m$ increases the variance of latent processes with Continuum Dropout, hindering stable training. Consequently, we limited our experiments to $m \in \lbrack 5, 10, 50, 100 \rbrack$ in this paper. This choice robustly improves performance regardless of hyperparameters $p$ and $m$.

\vspace{-5pt}
\begin{table}[H]
\scriptsize\centering\captionsetup{justification=centering, skip=5pt}
\caption{Accuracy on Speech Commands from the hyperparameter sensitivity analysis}
\label{tab:sensitivity}
\begin{tabular}{cccccc}
\Xhline{0.7pt}
\multirow{2}{*}{$p$} & \multicolumn{5}{c}{$m$} \\ \cline{2-6} 
                     & 3         & 5         & 10        & 50        & 100       \\ \hline
0.1 & 0.934 (0.002) & 0.938 (0.004) & 0.936 (0.002) &  0.938 (0.001) & 0.940 (0.002)
 \\
0.2 & 0.930 (0.003)  & 0.939 (0.002) & 0.938 (0.001) & 0.939 (0.002) & 0.936 (0.003) \\
0.3 & 0.921 (0.004)  & 0.936 (0.003) & 0.938 (0.003) & \textbf{0.940 (0.001)} & \textbf{0.940 (0.001)} \\
0.4 & \textcolor{red}{0.914 (0.005)}  & 0.931 (0.003) & 0.934 (0.003) & 0.933 (0.003) & 0.931 (0.003) \\
0.5 & \textcolor{red}{0.909 (0.008)}  & 0.928 (0.004) & 0.923 (0.003) & 0.927 (0.003) & 0.925 (0.004) \\ \Xhline{0.7pt}
\end{tabular}
\vspace{-5pt}
\end{table}

Figures~\ref{fig:hyperparam_cifar100}--\ref{fig:hyperparam_svhn} illustrate how the performance of proposed method for Neural ODE and Neural SDEs on image datasets varies with different hyperparameters $p$ and $m$, showing that the proposed method robustly enhances performance regardless of the choice of $p$ and $m$ on image datasets. Black dashed line and gray area indicate the mean and standard deviation of without Continuum Dropout.

\begin{figure*}[h]
    \centering\captionsetup{justification=centering, skip=5pt}
    \captionsetup[subfigure]{justification=centering, skip=5pt}
    \subfloat[Neural ODE]{
      \includegraphics[clip,width=0.32\linewidth]{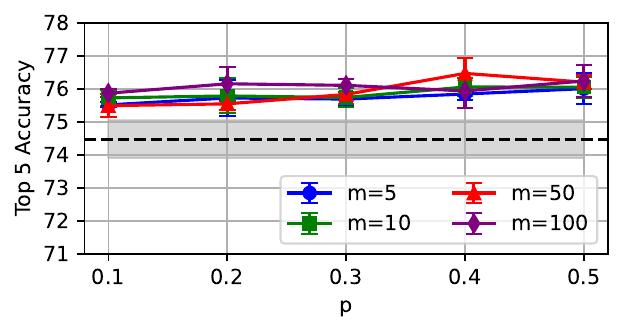}} \hfil
    \subfloat[Neural SDE (additive)]{
      \includegraphics[clip,width=0.32\linewidth]{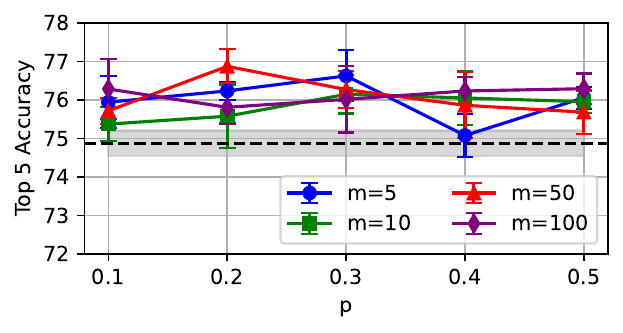}} \hfil
    \subfloat[Neural SDE (multiplicative)]{
      \includegraphics[clip,width=0.32\linewidth]{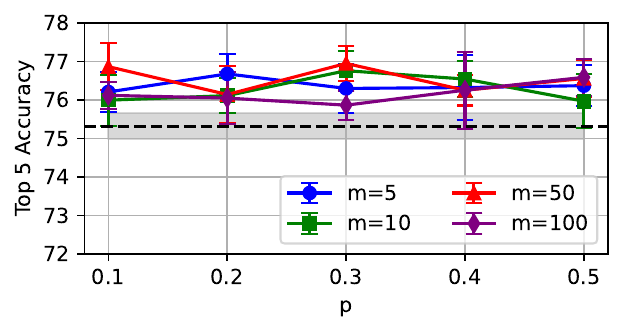}} 
    \caption{Performance comparison with different hyperparameters on CIFAR-100}\label{fig:hyperparam_cifar100}
    \vspace{-10pt}
\end{figure*}

\vspace{-5pt} 
\begin{figure*}[h]
    \centering\captionsetup{justification=centering, skip=5pt}
    \captionsetup[subfigure]{justification=centering, skip=5pt}
    \subfloat[Neural ODE]{
      \includegraphics[clip,width=0.32\linewidth]{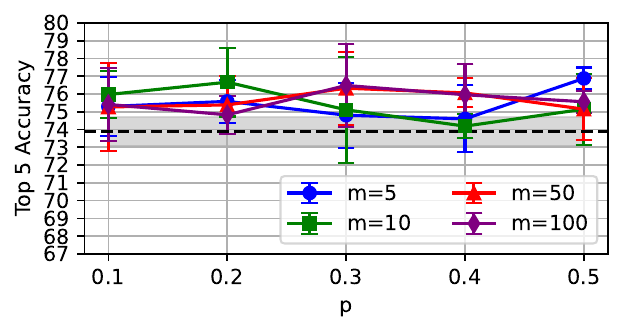}} \hfil
    \subfloat[Neural SDE (additive)]{
      \includegraphics[clip,width=0.32\linewidth]{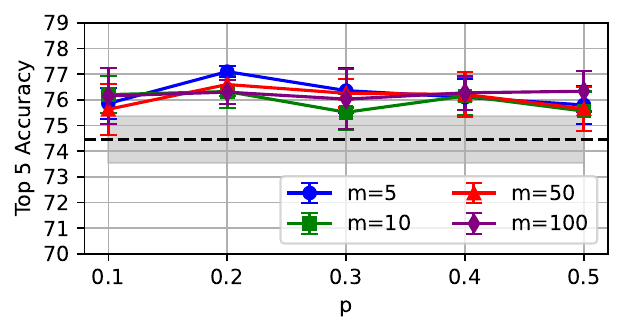}} \hfil
    \subfloat[Neural SDE (multiplicative)]{
      \includegraphics[clip,width=0.32\linewidth]{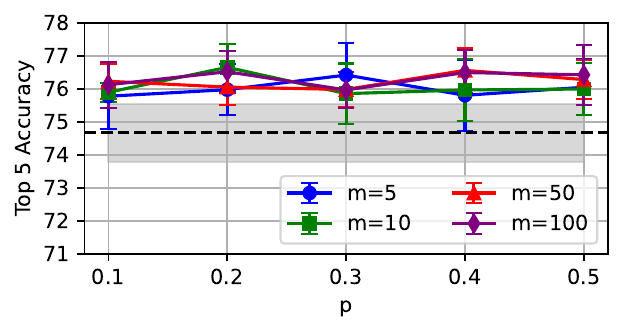}} 
    \caption{Performance comparison with different hyperparameters on CIFAR-10}\label{fig:hyperparam_cifar10}
    \vspace{-10pt}    
\end{figure*}

\begin{figure*}[h]
    \centering\captionsetup{justification=centering, skip=5pt}
    \captionsetup[subfigure]{justification=centering, skip=5pt}
    \subfloat[Neural ODE]{
      \includegraphics[clip,width=0.32\linewidth]{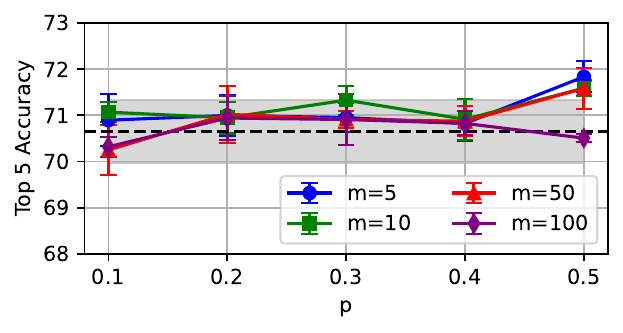}} \hfil
    \subfloat[Neural SDE (additive)]{
      \includegraphics[clip,width=0.32\linewidth]{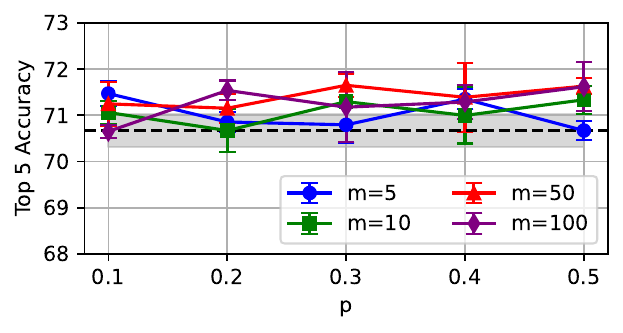}} \hfil
    \subfloat[Neural SDE (multiplicative)]{
      \includegraphics[clip,width=0.32\linewidth]{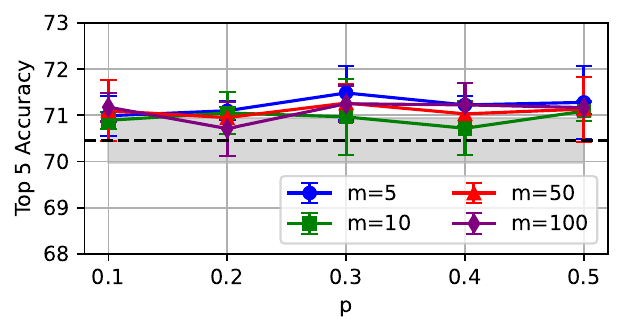}} 
    \caption{Performance comparison with different hyperparameters on STL-10} \label{fig:hyperparam_stl10}
    \vspace{-8pt}
\end{figure*}

\begin{figure*}[h]
    \centering\captionsetup{justification=centering, skip=5pt}
    \captionsetup[subfigure]{justification=centering, skip=5pt}
    \subfloat[Neural ODE]{
      \includegraphics[clip,width=0.32\linewidth]{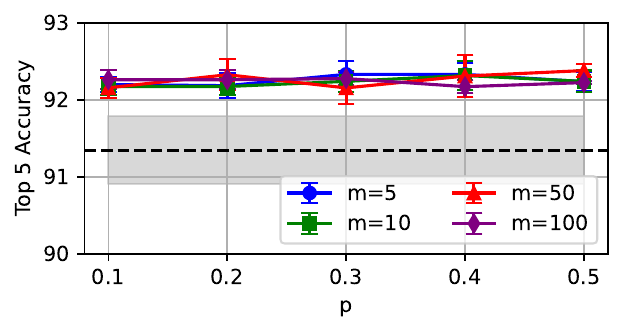}} \hfil
    \subfloat[Neural SDE (additive)]{
      \includegraphics[clip,width=0.32\linewidth]{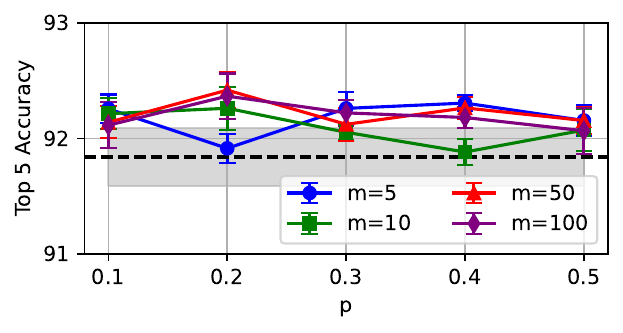}} \hfil
    \subfloat[Neural SDE (multiplicative)]{
      \includegraphics[clip,width=0.32\linewidth]{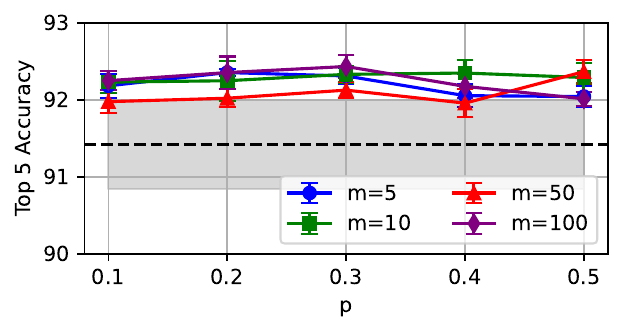}} 
    \caption{Performance comparison with different hyperparameters on SVHN}\label{fig:hyperparam_svhn}

\end{figure*}

\begin{figure}[h]
    \centering\captionsetup{justification=centering, skip=5pt}
    \captionsetup[subfigure]{justification=centering, skip=5pt}
    
    \begin{subfigure}{0.38\textwidth}
        \includegraphics[width=\linewidth]{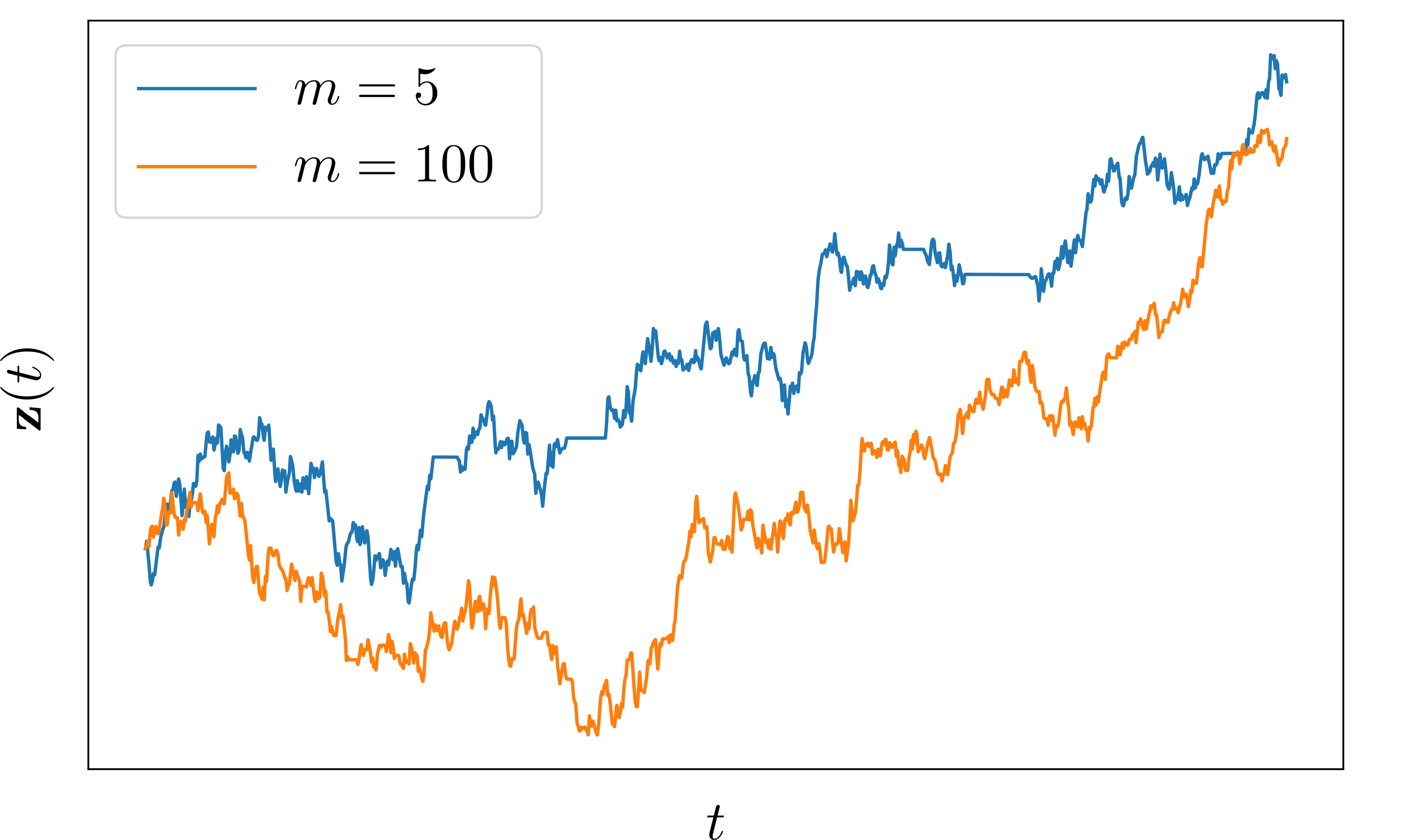}
        \caption{Dropout rate $p = 0.1$}
        \label{subfig:p1}
    \end{subfigure}
    \hspace{0.02\textwidth}
    \begin{subfigure}{0.38\textwidth}
        \includegraphics[width=\linewidth]{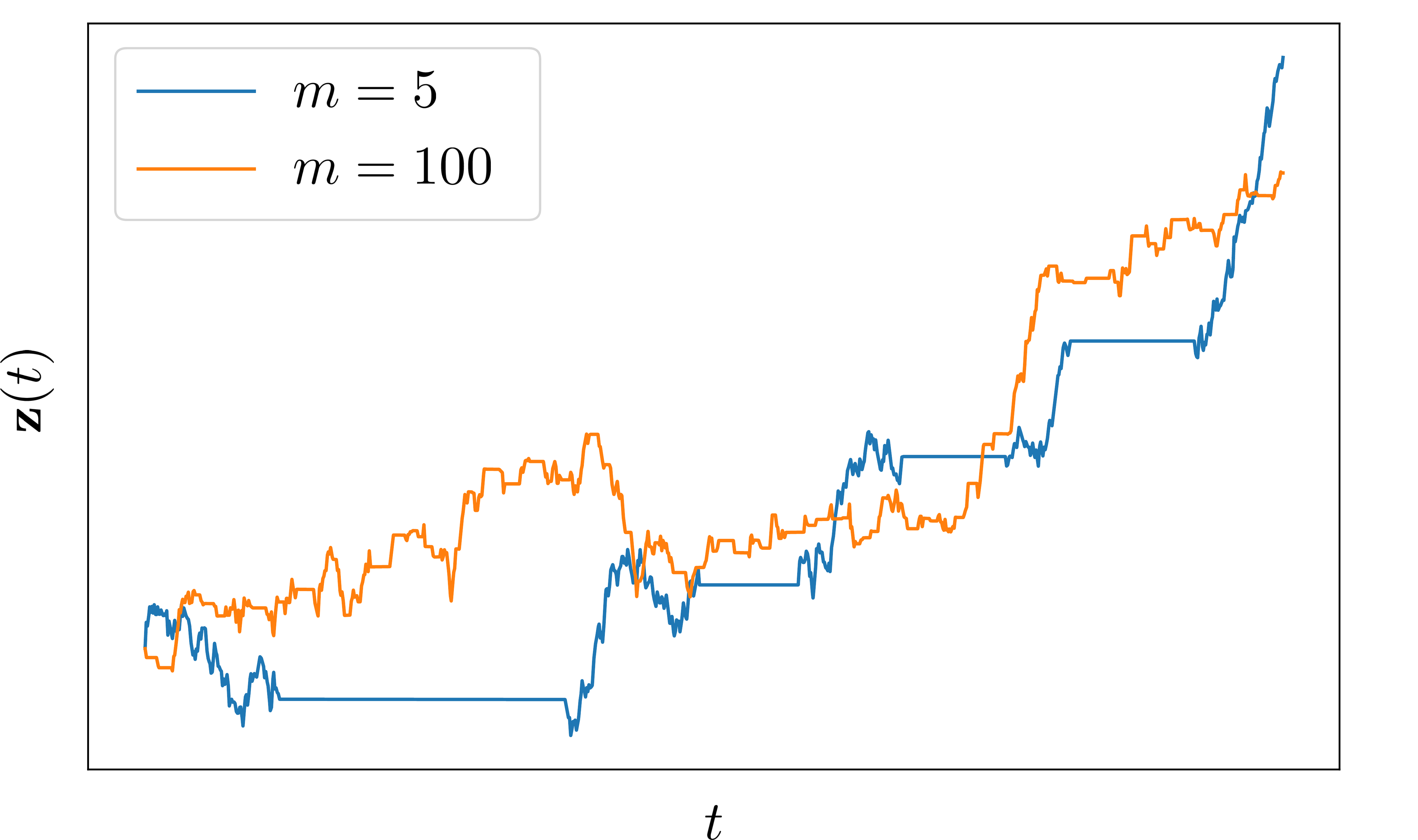}
        \caption{Dropout rate $p = 0.5$}
        \label{subfig:p5}
    \end{subfigure}

    \caption{Illustration of latent process $\rvz(t)$ with Continuum Dropout, using different hyperparameters $p$ and $m$}
    \label{fig:sample_diff_p}
\end{figure}

\subsection{Effect of Monte Carlo Sample Size}\label{app:c}

We designed experiments to analyze the effect of Monte Carlo sample size during the test phase. Table~\ref{tab:validation1} and Table~\ref{tab:validation2} show the optimal performance based on the number of Monte Carlo simulation samples $N_{\text{MC}}$ on Speech Commands and CIFAR-100, respectively. We considered $N_{\text{MC}} = [1, 3, 5, 10, 20]$, and observed consistent performance improvements with only $5$ samples, which supports the accurate inference. 

\begin{table}[htb]
\scriptsize\centering\captionsetup{justification=centering, skip=5pt}
\caption{Accuracy on Speech Commands with different numbers of Monte Carlo simulation samples}
\label{tab:validation1}
\begin{tabular}{@{\hspace{0.2cm}}ccccccc@{\hspace{0.2cm}}}
\Xhline{0.7pt}
 \multirow{2}{*}{\begin{tabular}[c]{@{}c@{}}\textbf{Continuum} \\ \textbf{Dropout}\end{tabular}}   & \multirow{2}{*}{\textbf{$N_{\text{MC}}$}} & \multirow{2}{*}{\textbf{Neural CDE}} & \multirow{2}{*}{\textbf{ANCDE}} & \multirow{2}{*}{\textbf{Neural LSDE}} & \multirow{2}{*}{\textbf{Neural LNSDE}} & \multirow{2}{*}{\textbf{Neural GSDE}} \\
 & & & & & & \\ \hline
 X & - & 0.910 (0.005) & 0.760 (0.003) & 0.927 (0.004) & 0.923 (0.001) & 0.913 (0.001) \\ \hline
 \multirow{5}{*}{O} 
 & 1 & 0.908 (0.002) & 0.762 (0.002) & 0.931 (0.003) & 0.925 (0.002) & 0.919 (0.006) \\
 & 3 & 0.929 (0.002) & 0.782 (0.006) & 0.932 (0.001) & 0.930 (0.000) & 0.927 (0.002)\\
 & 5 & 0.940 (0.001) & \textbf{0.794 (0.003)} & 0.932 (0.000) & \textbf{0.932 (0.001)} & 0.927 (0.001) \\
 & 10  & \textbf{0.945 (0.001)} & 0.793 (0.007) & \textbf{0.933 (0.001)} & 0.932 (0.002) & 0.930 (0.002) \\
 & 20  & 0.943 (0.003) & 0.793 (0.005) & 0.932 (0.001) & 0.932 (0.002) & \textbf{0.930 (0.001)} \\
\Xhline{0.7pt}
\end{tabular}
\end{table}

\begin{table}[!ht]
\scriptsize\centering\captionsetup{justification=centering, skip=5pt}
\caption{Top-5 accuracy on CIFAR-100 with different numbers of Monte Carlo simulation samples}
\label{tab:validation2}
\begin{tabular}{@{\hspace{0.2cm}}ccccc@{\hspace{0.2cm}}}
\Xhline{0.7pt}
 \multirow{2}{*}{\begin{tabular}[c]{@{}c@{}}\textbf{Continuum} \\ \textbf{Dropout}\end{tabular}}  & \multirow{2}{*}{\textbf{$N_{\text{MC}}$}} & \multirow{2}{*}{\textbf{Neural ODE}} & \multirow{2}{*}{\begin{tabular}[c]{@{}c@{}}\textbf{Neural SDE} \\      \textbf{(additive)}\end{tabular}} & \multirow{2}{*}{\begin{tabular}[c]{@{}c@{}}\textbf{Neural SDE} \\      \textbf{(multiplicative)}\end{tabular}} \\
 & & & & \\ \hline
 X & - & 74.475 (0.581) & 74.878 (0.328) & 75.317 (0.338) \\ \hline
 \multirow{5}{*}{O} 
 & 1 & 76.365 (0.342) & 76.390 (0.263) &  76.006 (0.426) \\
 & 3 & \textbf{76.848 (0.254)} & 76.547 (0.398) &  76.182 (0.522)  \\
 & 5 & 76.213 (0.238) & \textbf{76.958 (0.415)} & 76.787 (0.420) \\
 & 10  & 76.470 (0.480) & 76.872 (0.442) & \textbf{76.947 (0.458)} \\
 & 20  & 76.374 (0.298) & 76.704 (0.252) & 76.927 (0.404) \\
\Xhline{0.7pt}
\end{tabular}
\end{table}

\newpage
\subsection{Comprehensive Analysis of Computational Efficiency}\label{app:compute}
We performed comprehensive experiments to evaluate the computational overhead of our proposed method (Continuum Dropout) with varying dropout rates $p$ and expected renewal counts $m$. Table~\ref{tab:computation} presents the average computation times per epoch on Speech Commands dataset. The results indicate that while the proposed method generally increases computation time compared to the baseline, the impact of different $p$ and $m$ values is relatively minor. Also, since Monte Carlo simulation is only conducted during the test phase, we observe that it does not significantly increase the computational overhead. The computational overhead remains almost negligible, showing only a marginal increase compared to the baseline. 

\begin{table}[htb]
\scriptsize\centering\captionsetup{justification=centering, skip=5pt}
\caption{Computation time comparison on Speech Commands (time in seconds per epoch)}\label{tab:computation}
\begin{minipage}[b]{\textwidth}
    \centering
    \begin{tabular}{@{\hspace{0.2cm}}ccccccc@{\hspace{0.2cm}}} 
    \Xhline{0.7pt}
    $p$ & $m$ & \textbf{Neural CDE} & \textbf{ANCDE} & \textbf{Neural LSDE} & \textbf{Neural LNSDE} & \textbf{Neural GSDE} \\ 
    \hline
    0   & -   & 25.560 (0.259) & 53.264 (0.157) & 19.416 (0.147) & 19.532 (0.109) & 19.776 (0.052)  \\ 
    \hline
    \multirow{4}{*}{0.1} & 5   & 28.056 (0.262) & 60.208 (0.292) & 22.086 (0.158) & 22.499 (0.223) & 23.095 (0.233)  \\ 
        & 10  & 28.182 (0.258) & 59.914 (0.199) & 22.076 (0.171) & 22.207 (0.173) & 22.807 (0.112)  \\ 
        & 50  & 28.014 (0.245) & 59.339 (0.307) & 22.157 (0.150) & 22.213 (0.161) & 23.178 (0.246)  \\ 
        & 100 & 28.375 (0.508) & 60.100 (0.340) & 22.332 (0.190) & 22.128 (0.122) & 22.939 (0.131)  \\ 
    \hline
    \multirow{4}{*}{0.3} & 5   & 27.180 (0.366) & 59.148 (0.408) & 21.576 (0.155) & 21.668 (0.171) & 22.496 (0.176)  \\ 
        & 10  & 27.226 (0.305) & 59.156 (0.367) & 21.610 (0.135) & 21.632 (0.144) & 22.460 (0.216)  \\ 
        & 50  & 27.064 (0.248) & 59.062 (0.373) & 21.589 (0.214) & 21.646 (0.199) & 22.482 (0.183)  \\ 
        & 100 & 27.145 (0.255) & 59.093 (0.276) & 21.580 (0.214) & 21.740 (0.131) & 22.436 (0.174)  \\ 
    \hline
    \multirow{4}{*}{0.5} & 5   & 26.052 (0.255) & 58.108 (0.229) & 21.045 (0.094) & 20.790 (0.174) & 21.818 (0.089)  \\ 
        & 10  & 26.393 (0.276) & 57.643 (0.136) & 21.058 (0.097) & 21.185 (0.091) & 21.893 (0.088)  \\ 
        & 50  & 25.688 (0.219) & 57.646 (0.407) & 21.025 (0.135) & 21.176 (0.120) & 21.931 (0.116)  \\ 
        & 100 & 26.144 (0.314) & 57.905 (0.265) & 21.005 (0.104) & 21.098 (0.144) & 21.960 (0.084)  \\ 
    \Xhline{0.7pt}
    \end{tabular}
\end{minipage}
\vspace{5pt}
\end{table}

\section{Limitations of Jump Diffusion Model}\label{app:why?}


In this section, we examine the limitations of the jump diffusion model discussed in the main text in greater detail. Specifically, Appendix~\ref{app:discretization} analyzes the theoretical issues of jump diffusion model, Appendix~\ref{app:liu_hyperparameter} provides an analysis of the practical 
difficulty in tuning the dropout rate $p$, which is a hyperparameter of jump diffusion model, and Appendix~\ref{app:why_cannot} demonstrates, through experiments, why jump diffusion model cannot be universally applied to various variants of NDEs.

\subsection{Discretization}\label{app:discretization}
\citet{liu2020does} claimed that 
\begin{equation}
\begin{aligned}
    \rvZ(k+1) &= \rvZ(k) + \gamma(\rvZ(k); \theta_\gamma) \circ \xi \\
    &= \rvZ(k) + \frac{1}{2} \gamma(\rvZ(k); \theta_\gamma) + \frac{1}{2}\gamma(\rvZ(k); \theta_\gamma) \circ \Xi
\end{aligned}
\label{eq:liu_discrete2}
\end{equation}
with $k=0,1,\ldots, N-1$, $\sP(\xi^{(i)}=0)=1-\sP(\xi^{(i)}=1)=p$ for $i=1,\ldots, d_z$ and $\Xi = 2\xi - 1$ is a discrete version of the following jump diffusion process: for $0\leq t \leq T$,
\begin{align}
    \rvz(t) = \rvz(0) + \int_0^t \frac{1}{2} \gamma(\rvz(\tau); \theta_\gamma) \, \mathrm{d}\tau + \int_0^t \frac{1}{2} \gamma(\rvz(\tau); \theta_\gamma) \circ \Xi_{N_\tau} \, \mathrm{d}N_\tau,  \label{eq:liu_jd}
\end{align}
where $N_\tau$ is a Poisson counting process.

However, that claim is not only incomplete in defining \eqref{eq:liu_discrete2}, but even if we were to correct it properly, 
our analysis reveals theoretical inconsistencies in this approach. 
More specifically, the time step size $\Delta t$ has not been considered in \eqref{eq:liu_discrete2}, which is essential for it to be a valid discrete approximation of the continuous jump diffusion process. In fact, a correct Euler discretization of \eqref{eq:liu_jd} is given by 
\begin{align*}
    \rvZ(k+1) &= \rvZ(k) + \frac{1}{2} \gamma(\rvZ(k); \theta_\gamma) \Delta t + \frac{1}{2} \gamma(\rvZ(k); \theta_\gamma) \circ \Xi_{N_k} \Delta N_k\\
    &= \rvZ(k) + 
    \begin{cases}
    \begin{aligned}
        &\frac{1}{2} \gamma(\rvZ(k); \theta_\gamma) \Delta t + \frac{1}{2} \gamma(\rvZ(k); \theta_\gamma) \circ \Xi \vphantom{\frac{1}{2}}, & \text{if } \Delta N_k = 1,\\[0.5em]
        &\frac{1}{2} \gamma(\rvZ(k); \theta_\gamma) \Delta t \vphantom{\frac{1}{2}}, & \text{if } \Delta N_k = 0,
    \end{aligned}
    \end{cases}\\
    &\neq \rvZ(k) + \gamma(\rvZ(k); \theta_\gamma) \circ \xi, 
\end{align*}
where  $\Delta N_k = N_{\frac{T}{N}k} - N_{\frac{T}{N}k^-} \in \{0, 1\}$.

In contrast, our proposed method based on the alternating renewal processes (Continuum Dropout) accurately extends the following discrete-time equation to the continuous-time process:
\begin{equation*}
    \rvZ(k+1) = \rvZ(k) + \gamma(\rvZ(k); \theta_\gamma)  \Delta t \circ I_k.
\end{equation*}

\subsection{Hyperparameter Tuning}\label{app:liu_hyperparameter}
\vspace{-1pt}
We evaluated the performance of jump diffusion model with respect to the hyperparameter dropout rate \( p \). Tables \ref{tab1} and \ref{tab2} present the results of classification tasks on time series and image datasets, respectively. Jump diffusion model showed performance improvement at very low dropout rates below 0.1, but no significant improvement was observed beyond that. This behavior differs from na\"{\i}ve dropout, and it introduces additional cost for tuning. However, our proposed method (Continuum Dropout), as analyzed in Appendix~\ref{sec:sensitivity}, consistently improves performance robustly within the tuning grid of na\"{\i}ve dropout.

\vspace{-5pt}
\begin{table}[H]
\scriptsize\centering\captionsetup{justification=centering, skip=5pt}
\caption{Results of hyperparameter tuning for jump diffusion model in time series datasets}\label{tab1}
\begin{minipage}[b]{0.72\textwidth}
    \centering
    \caption*{(a) Jump Diffusion \cite{liu2020does}}
    \begin{tabular}{ccccc}
    \Xhline{0.7pt}
    $p$ & \textbf{SmoothSubspace} & \textbf{ArticularyWordRecognition} & \textbf{ERing} & \textbf{RacketSports} \\ \hline
    $0$ & 0.569 (0.040) & 0.859 (0.005) & 0.839 (0.018) & 0.565 (0.065) \\ \hline
    $10^{-5}$ & 0.594 (0.048) & \textbf{0.871 (0.054)} & 0.844 (0.050)  & 0.571 (0.018)  \\ 
    $10^{-4}$ & \textbf{0.617 (0.043)} & 0.862 (0.043) & \textbf{0.861 (0.064)}  & 0.592 (0.032)   \\
    $10^{-3}$ & 0.606 (0.043) & 0.862 (0.014)  & 0.850 (0.036) &  \textbf{0.609 (0.031)}  \\
    $10^{-2}$ & 0.594 (0.036) & 0.856 (0.035)   & 0.850 (0.055)  & 0.592 (0.056)  \\
    $0.1$ & 0.600 (0.057) & 0.859 (0.021)  & 0.844 (0.057) & 0.609 (0.061)  \\ 
    $0.2$ & 0.611 (0.011)  & 0.853 (0.029) & 0.844 (0.065) & 0.576 (0.033) \\
    $0.3$ & 0.589 (0.060) & 0.862 (0.049) &0.856 (0.040)   & 0.582 (0.071)   \\
    $0.4$ & 0.594 (0.024) & 0.856 (0.014) & 0.856 (0.056)  & 0.576 (0.036)  \\
    $0.5$ & 0.600 (0.068) & 0.862 (0.017) & 0.850 (0.033)  & 0.569 (0.041) \\ \Xhline{0.7pt}
    \end{tabular}
\end{minipage}%
\vspace{4pt}
\begin{minipage}[b]{0.72\textwidth}
    \centering
    \caption*{(b) Jump Diffusion + TTN \cite{liu2020does}}
    \begin{tabular}{ccccc}
    \Xhline{0.7pt}
    $p$ & \textbf{SmoothSubspace} & \textbf{ArticularyWordRecognition} &\textbf{ERing} & \textbf{RacketSports} \\ \hline
    $0$ & 0.569 (0.040) & 0.859 (0.005) & 0.839 (0.018) & 0.565 (0.065) \\ \hline
    $10^{-5}$ & 0.583 (0.036) &  0.874 (0.029) & 0.872 (0.010) & 0.587 (0.080)  \\ 
    $10^{-4}$ & 0.594 (0.024) & \textbf{0.876 (0.025)} & \textbf{0.878 (0.025)} & 0.576 (0.024)  \\
    $10^{-3}$ & 0.583 (0.024) & 0.874 (0.022) & 0.867 (0.031)& \textbf{0.598 (0.076)}  \\
    $10^{-2}$ & 0.594 (0.040)  & 0.874 (0.031) & 0.861 (0.018) & 0.576 (0.059)  \\
    $0.1$ & \textbf{0.606 (0.018)} &  0.871 (0.015) &  0.872 (0.033) & 0.560 (0.060) \\ 
    $0.2$ & 0.594 (0.065)  & 0.862 (0.021) &  0.872 (0.036) & 0.587 (0.060)  \\
    $0.3$ & 0.600 (0.057) &  0.868 (0.031)  &  0.850 (0.024)& 0.572 (0.051)  \\
    $0.4$ & 0.606 (0.048)  &  0.871 (0.033) & 0.861 (0.033)  & 0.571 (0.056) \\
    $0.5$ & 0.600 (0.010) & 0.862 (0.020)  & 0.844 (0.016)  & 0.578 (0.039)  \\ \Xhline{0.7pt}
    \end{tabular}
\end{minipage}%
\end{table}
\vspace{-12pt}
\begin{table}[H]
\scriptsize\centering\captionsetup{justification=centering, skip=5pt}
\caption{Results of hyperparameter tuning for jump diffusion model in image datasets}\label{tab2}
\begin{minipage}[b]{0.50\textwidth}
    \centering
    \caption*{(a) Jump Diffusion \cite{liu2020does}}
    \begin{tabular}{ccccc}
    \Xhline{0.7pt}
    $p$ & \textbf{CIFAR-100} & \textbf{CIFAR-10} & \textbf{STL-10} & \textbf{SVHN} \\ \hline
    $0$ & 74.475 (1.181) & 73.870 (0.820) & 70.650 (0.688)  & 91.348 (0.440)\\ \hline
    $10^{-5}$ & 74.852 (1.085) & \textbf{74.987 (0.350)} & \textbf{71.097 (0.242)} & 91.388 (0.348)  \\ 
    $10^{-4}$ & \textbf{76.083 (0.502)} & 74.370 (1.689) &  71.084 (0.215) & 90.906 (0.422)  \\
    $10^{-3}$ & 73.873 (1.973) & 74.377 (1.303)  & 70.912 (0.436) & 91.520 (0.450)  \\
    $10^{-2}$ & 75.098 (0.530) & 74.600 (1.165) & 70.409 (0.370)
 & 91.194 (0.556)  \\
    $0.1$ & 75.415 (0.415) & 73.570 (1.788) & 70.966 (0.192) & 91.272 (0.367) \\ 
    $0.2$ & 74.975 (0.394) & 72.927 (1.782) & 70.931 (0.230) & 91.539 (0.238) \\
    $0.3$ & 74.782 (0.608) & 72.727 (1.437)  & 70.938 (0.648) & 91.478 (0.146)  \\
    $0.4$ & 74.900 (1.401) & 72.282 (1.413) & 70.903 (0.364) & \textbf{91.568 (0.413)} \\
    $0.5$ & 74.535 (0.231) & 73.850 (2.184) & 70.778 (0.204) & 91.480 (0.135) \\ \Xhline{0.7pt}
    \end{tabular}
\end{minipage}%
\vspace{4pt}
\begin{minipage}[b]{0.50\textwidth}
    \centering
    \caption*{(b) Jump Diffusion + TTN \cite{liu2020does}}
    \begin{tabular}{ccccc}
    \Xhline{0.7pt}
    $p$ & \textbf{CIFAR-100} & \textbf{CIFAR-10} & \textbf{STL-10} & \textbf{SVHN} \\ \hline
    $0$ & 74.475 (1.181) & 73.870 (0.820) & 70.650 (0.688)  & 91.348 (0.440)\\ \hline
    $10^{-5}$ & \textbf{76.013 (0.276)} & 73.742 (1.646) & 70.747 (0.479) & \textbf{91.730 (0.518)} \\ 
    $10^{-4}$ & 75.743 (0.307) & 74.305 (2.027) & 70.716 (0.282) & 91.366 (0.110)  \\
    $10^{-3}$ & 75.227 (0.261) & 74.755 (1.500) & 70.409 (0.390) & 91.323 (0.566)  \\
    $10^{-2}$ & 74.430 (1.637) & 74.710 (1.307)  & 70.544 (0.177) & 91.088 (0.309)   \\
    $0.1$ & 75.170 (0.213) & \textbf{75.015 (0.503)} & 70.775 (0.259)  & 91.509 (0.107) \\ 
    $0.2$ & 75.692 (0.392) & 73.403 (1.190) & \textbf{70.931 (0.286)} & 91.561 (0.494)  \\
    $0.3$ & 74.992 (1.470) & 73.375 (2.269) &  70.819 (0.455)& 91.267 (0.360)  \\
    $0.4$ & 74.553 (1.220) & 73.915 (1.692)  &  70.741 (0.206)& 91.423 (0.401) \\
    $0.5$ & 75.502 (1.176) & 73.160 (1.698) & 70.722 (0.411) &91.238 (0.606)   \\ \Xhline{0.7pt}
    \end{tabular}
\end{minipage}
\end{table}

\subsection{Universal Applicability}\label{app:why_cannot}

We discuss why jump diffusion model is not universally applicable, unlike the proposed method (Continuum Dropout). \citet{liu2020does} introduced a new type of Neural SDE by adding a stochastic dropout term to Neural ODE. However, this approach is limited to ODE-based models. Nonetheless, from an engineering perspective, we compared the performance of the proposed method with the model, which combines jump diffusion noise with the CDE-based model. Additionally, in the case of SDE-based models, the existing diffusion network dominates the jump diffusion model term, and thus, these models are not considered. 



Table~\ref{tab:cdeliu} presents the performance of jump diffusion model and our proposed method on Speech Commands, using Neural CDE as the baseline. Under the jump‑diffusion model, performance either declines or, if it improves, the improvement is not statistically significant. However, our proposed method demonstrates statistically significant and highly successful performance improvements, experimentally proving that it is the only universally applicable dropout method.

\begin{table}[H]
\scriptsize\centering\captionsetup{justification=centering, skip=5pt}
\caption{Performance of various dropout methods on Speech Commands}\label{tab:cdeliu}
\begin{tabular}{cc}
\Xhline{0.7pt}
\textbf{Dropout Methods}       & \textbf{Test Accuracy}   \\ \hline
Baseline (Neural CDE)   & 0.910 (0.005)           \\ \hline
Jump Diffusion \cite{liu2020does}   &  0.917 (0.009)         \\
Jump Diffusion + TTN \cite{liu2020does} &  0.906 (0.013)           \\  \hline
\textbf{Continuum Dropout} (ours)     & \textbf{\phantom{$^{**}$}0.940 (0.001)$^{**}$} \\ \Xhline{0.7pt}
\end{tabular}
\end{table}


\end{document}